\pdfoutput=1
\documentclass[11pt]{article}

\usepackage[final]{acl}

\usepackage{times}
\usepackage{latexsym}
\usepackage{enumitem}

\usepackage[T1]{fontenc}
\usepackage[utf8]{inputenc}

\usepackage{microtype}

\usepackage{inconsolata}

\usepackage{graphicx}

\usepackage{microtype}
\usepackage{graphicx}
\usepackage{subfigure}
\usepackage{booktabs} % for professional tables

\usepackage{hyperref}

\usepackage{algorithm}
\usepackage{algorithmicx}
\usepackage{algpseudocode}

\usepackage{amsmath}
\usepackage{amssymb}
\usepackage{mathtools}
\usepackage{amsthm}

\usepackage[capitalize,noabbrev]{cleveref}

\def\1{\bm{1}}

\def\mD{{\bm{D}}}

\def\mI{{\bm{I}}}

\DeclareMathAlphabet{\mathsfit}{\encodingdefault}{\sfdefault}{m}{sl}
\SetMathAlphabet{\mathsfit}{bold}{\encodingdefault}{\sfdefault}{bx}{n}

\newcommand{\E}{\mathbb{E}}

\newcommand{\R}{\mathbb{R}}

\DeclareMathOperator*{\argmin}{arg\,min}

\usepackage{pifont}

\usepackage{hyperref}
\usepackage{multirow}
\usepackage{colortbl}
\definecolor{niceblue}{rgb}{0.0,0.19,0.56}
\usepackage{url}
\usepackage[utf8]{inputenc} % allow utf-8 input
\usepackage[T1]{fontenc}    % use 8-bit T1 fonts

\usepackage{url}            % simple URL typesetting
\usepackage{booktabs}       % professional-quality tables
\usepackage{amsfonts}       % blackboard math symbols
\usepackage{nicefrac}       % compact symbols for 1/2, etc.
\usepackage{microtype}      % microtypography
\usepackage{bbm}            % for indicators 1
\usepackage{bm}             % bold symbols

\usepackage{eqparbox}
\newcommand{\cD}{{\cal D}}

\newcommand{\cG}{{\cal G}}

\newcommand{\cN}{{\cal N}}

\newcommand\swapifbranches[3]{#1{#3}{#2}}
\makeatletter
\MHInternalSyntaxOn
\patchcmd{\DeclarePairedDelimiter}{\@ifstar}{\swapifbranches\@ifstar}{}{}
\MHInternalSyntaxOff
\makeatother

\DeclarePairedDelimiterX{\inp}[2]{\langle}{\rangle}{#1, #2}
\DeclarePairedDelimiterX{\abs}[1]{\lvert}{\rvert}{#1}
\DeclarePairedDelimiterX{\roundup}[1]{\lceil}{\rceil}{#1}
\DeclarePairedDelimiterX{\norm}[1]{\lVert}{\rVert}{#1}
\DeclarePairedDelimiterX{\cbr}[1]{\{}{\}}{#1} % curly bracket
\DeclarePairedDelimiterX{\rbr}[1]{(}{)}{#1} % round bracket
\DeclarePairedDelimiterX{\sbr}[1]{[}{]}{#1} % 

\newtheorem{theorem}{Theorem}
\newcommand{\f}{f}
\newcommand{\xiv}{\xi}
\newcommand{\g}{g}
\newcommand{\x}{x}
\newcommand{\y}{y}
\newcommand{\w}{w}
\newcommand{\n}{{n}}

\usepackage{xspace}
\newcommand{\algname}[1]{{\texttt{#1}}\xspace}

\newtheorem{assumption}{Assumption}

\usepackage[colorinlistoftodos,bordercolor=orange,backgroundcolor=orange!20,linecolor=orange,textsize=scriptsize]{todonotes}

\title{Low-Resource Machine Translation through the Lens of\\ Personalized Federated Learning}

\author{
  \textbf{Viktor Moskvoretskii\textsuperscript{1, 4}},
  \textbf{Nazarii Tupitsa\textsuperscript{2, 5, 6}},
  \textbf{Chris Biemann\textsuperscript{3}},
  \textbf{Samuel Horváth\textsuperscript{2}},
\\
  \textbf{Eduard Gorbunov\textsuperscript{2}},
  \textbf{Irina Nikishina\textsuperscript{3}}
\\
\\
  \textsuperscript{1}Skoltech,
  \textsuperscript{2}MBZUAI,
  \textsuperscript{3}Universität Hamburg,
  \textsuperscript{4}HSE University,
  \textsuperscript{5}MIPT, 
  \textsuperscript{6}Innopolis University
\\
  \small{
    \textbf{Correspondence:} \href{mailto:v.moskvoretskii@skol.tech}{v.moskvoretskii@skol.tech}
  }
}

\begin{document}
\maketitle
\begin{abstract}
We present a new approach called \algname{MeritOpt} based on the Personalized Federated Learning algorithm \algname{MeritFed} that can be applied to Natural Language Tasks with heterogeneous data. We evaluate it on the Low-Resource Machine Translation task, using the datasets of South East Asian and Finno-Ugric languages. In addition to its effectiveness, \algname{MeritOpt} is also highly interpretable, as it can be applied to track the impact of each language used for training. Our analysis reveals that target dataset size affects weight distribution across auxiliary languages, that unrelated languages do not interfere with the training, and auxiliary optimizer parameters have minimal impact. Our approach is easy to apply with a few lines of code, and we provide scripts for reproducing the experiments.\footnote{\url{https://github.com/VityaVitalich/MeritOpt}}
\end{abstract}

\section{Introduction}
While 7,000+ languages are currently in use worldwide, most existing Natural Language Processing (NLP) tasks and Large Language Models (LLMs) cover at most 500 of them \cite{DBLP:conf/lrec/LogachevaTSRUKA20,imanigooghari-etal-2023-glot500,DBLP:journals/corr/abs-2401-13303}. Many languages still possess low amount of resources, and a lot of NLP tasks for such languages remain unsolved. These facts indicate the difficulty and non-triviality of using LLMs that typically require large amounts of data. A popular direction of approaching low-resource languages (LRLs) is Machine Translation: automatic translation between most of these low-resource languages to high-resource ones is more economically and socially motivated than developing language-specific systems \cite{DBLP:journals/csur/RanathungaLSSAK23}.

To solve the tasks for LRLs, a lot of studies employ the related languages or languages originating from the same geographical and historical background \cite{imanigooghari-etal-2023-glot500,da-dalt-etal-2024-flor-effectiveness,millour-etal-2024-agettivu-aggitivu}. Despite the positive effect, it usually requires empirical knowledge, and many guesses and trials of different approaches when choosing the best combination of languages used, the most suitable amount of data, and the best learning strategy \cite{hedderich-etal-2021-survey}.

\paragraph{New approach.} To address these issues, we present our approach called \algname{MeritOpt} to train LLMs for the target language while multiple datasets for different languages are available. 
The key idea behind our method is inspired by \citet{tupitsa2024federated}, who focus on a specific (Personalized) Federated Learning formulation \citep{kairouz2021advances}. The method from \citet{tupitsa2024federated} is a special case of our method. We emphasize that the idea is borrowed from the Federated Learning (FL) field, however, no FL itself is applied in the paper. 
FL focuses on the specific setting of distributed training, when there exist multiple clients with their own (and private) data. In the scenario of  \citet{tupitsa2024federated}, versions of the Federated Averaging algorithm are natural choices for solving the problem since collecting raw data from the clients is prohibited due to privacy constraints. In contrast, we do not have clients or distributed systems for the problem we are considering. We focus on exploring the underlying algorithmic techniques in application to heterogeneous datasets rather than the Distributed Training. We are not restricted by any privacy constraints since the datasets we consider are open. That is, our work is not an FL paper.

% Therefore, our work is not a Federated Learning paper. Our approach adapts the method of \citet{tupitsa2024federated} for Natural Language Processing tasks and also is much more general. Nevertheless, we present a natural analogy between these two seemingly unrelated setups: although we do not have workers (clients), we can interpret each dataset for some language as a client, and we can also interpret the languages itself as some data-distributions of those clients. We believe that this viewpoint (analogy) is interesting on its own and opens a giant room for future research in NLP. Moreover, our paper indicates that this direction is indeed prominent: we adjusted one particular FL algorithm to the setting of training LLMs for low-resource languages, which are not directly related to FL, and showed promising results. 

%This was also non-trivial beforehand since the analysis of SGD is quite different from the analysis of both RMSProp and AdaGrad (see Appendix D and Zaheer et al. (2018) and Ward et al. (2019) from our reference list for more details). Finally and perhaps more importantly, we emphasize that Tupitsa et al. (2024) consider a completely different problem setup, as we explain in the response to W2.

 %Federated/Distributed Training procedure.

%The inductive bias of the proposed approach is that it does not require any prior knowledge about the considered datasets or languages and adjusts the impact of each language (\emph{aggregation weights}) during training.
Our approach is also more robust than the existing baselines as it adjusts the impact of each language (\emph{aggregation weights}) during training without any explicit inductive bias towards language relatedness.
%adjusts on the fly the \emph{aggregation weights} for combining the updates related to each dataset/language. 
In particular, our strategy benefits from the updates from the ``important'' languages and tolerates the updates from the ``not important'' ones. 
This setup is extremely beneficial for the interpretability of the training process.

%We also show how our method can be applied to any learning task (not necessarily related to NLP) where multiple heterogeneous input datasets are available, and the goal is to train the model suitable for some target data distribution. 

In this study, we primarily focus on low-resource languages. However, our approach can be applied to any similar task (not necessarily in NLP). The main requirement is to possess multiple heterogeneous input datasets, while the goal is to train the model suitable for some target data distribution. 

Therefore, we apply the algorithm to the Machine Translation task using two datasets: the subset from the Large-Scale Multilingual Machine Translation Shared Task (Small Track \#2) \cite{wenzek-etal-2021-findings} and the subset of Sami languages from the multilingual benchmark for Finno-Ugric languages \cite{yankovskaya-etal-2023-machine}. To test the method effectively within our compute budget, we focus our study on scenarios with one target language and the remaining languages as auxiliary languages. Our approach can be further applied to the datasets with several target languages and several translation directions. 

% For example, given Javanese, which can be considered low-resource, we fine-tune an LLM that translates Javanese into English using parallel data from the related Indonesian and Malay that have larger resources. We do not only outperform the multilingual baselines but also measure the impact of each language during training. %The example in Figure \ref{fig:example_java_small} demonstrates the evolution of aggregation weights between Javanese (target language) and additional languages (Malay, Tamil, Indonesian, Tagalog) across the training steps.

% In particular, we show how our approach can be applied to the training of LLMs for low-resource languages and potentially benefit from the usage of other related languages.

% To address these issues, we present our approach that can be applied to such data-heterogeneous scenarios and demonstrate better results than previous transfer and fine-tuning approaches. It \q{Edik, please help}
% Based on the Personalized Federated Learning technique \cite{}%https://arxiv.org/abs/2002.07948
% , it regards each language separately and automatically identifies the impact for each language subset to the learning process. Moreover, this strategy prevents the model from overfitting and ... \q{Edik, please help}. 

Two research questions are addressed in this paper: (i) \textit{``Can \algname{MeritOpt} improve the results of the multilingual or single language baselines using aggregation weights?''} and (ii) \textit{``How do the target language weights and the weights of related and non-related languages change across training?''}.

% \noindent\textbf{Our contributions} are summarized below.
The contributions of the paper are as follows:

\begin{itemize}
    \item We present a new algorithmic framework for the training from heterogeneous input datasets and test it on the Indonesian languages and Sami languages of the Finno-Ugric group.
    % on the South East Asian languages and Sami languages of the Finno-Ugric MT benchmark. %several language families.
    \item We explore how languages interact with each other during training, as our approach allows measuring the impact (which language contributes more) at each training step.

%    \item We test the developed approach on the World Machine Translation Shared task on the Indonesian languages and Sami languages of the Finno-Ugric MT benchmark.  % REDUNDANT
    \item We perform an ablation study to analyze the effects of unrelated languages, training dataset size, and auxiliary \algname{MeritOpt} parameters.
    \item Under certain assumptions, we rigorously prove that the proposed method converges to some neighborhood of the solution.
    % \item Under some regularity assumptions on the loss function of the target distribution and corresponding optimizer, we rigorously prove that the proposed method converges to some neighborhood of the solution of the target task.

    \item Finally, we present a natural analogy between two seemingly unrelated setups -- Federated Learning and LLMs training on low-resource languages: although we do not have workers (clients) in the later setup, we can interpret each dataset for some language as a client, and we can also interpret the languages itself as some data-distributions of those clients. We believe that this viewpoint/analogy is interesting on its own and opens a giant room for future research in NLP. Moreover, our paper indicates that this direction is indeed prominent: we adjusted one particular FL algorithm to the setting of training LLMs for low-resource languages, which are not directly related to FL, and showed promising results.
\end{itemize}

\section{Related Work}

In this section, we discuss the existing methods for low-resource language NLP tasks, especially for low-resource machine translation (LRMT) \cite{10.1162/coli_a_00446}, and also give a brief overview of the existing methods in Personalized Federated Learning, and methods to estimate the impact of auxiliary data.

Regarding similar approaches, the paper of \citet{wang-etal-2020-balancing} also assigns the non-uniform weights for different languages. However, we do not compute any gradient similarity metrics and approximately solve an auxiliary problem to find the aggregation weights.

\subsection{Low-Resource Machine Translation}

Existing approaches for NLP tasks for LRLs usually fall into the following categories: supervised or unsupervised, single language training or multilingual training, continuous pre-training or finetuning, with or without data augmentation, balanced or imbalanced datasets \cite{hedderich-etal-2021-survey,DBLP:conf/ijcai/Wang0LQL21,electronics13030648,goyal-etal-2020-efficient}. This list of categories is not extensive. However, they all aim to develop the best learning strategy given limited data.

In the following subsections, we discuss the methods developed or applied for the datasets on South East Asian Languages and Finno-Urgic benchmarks, the main targets of our research.

\subsubsection{LRMT for South East Asian Languages}

Several approaches have been developed to solve the Large-scale Multilingual Machine Translation task (Shared Task on WMT-21). 
The organizers \cite{wenzek-etal-2021-findings} summarize all the used approaches and provide the FLORES model  \cite{goyal-etal-2022-flores} extended to 124 languages. 
Most of the participants, \citet{yang-etal-2021-multilingual-machine,budiwati-etal-2021-optimize,liao-etal-2021-back}, use a generic pre-trained multilingual models like DeltaLM \cite{DBLP:journals/corr/abs-2106-13736} or FLORES \cite{goyal-etal-2022-flores}
and fine-tune it correspondingly with the vast collected parallel data, together with applying progressive learning and iterative back-translation. 
%\citet{xie-etal-2021-tentrans} utilize forward/back-translation, in-domain data selection, knowledge distillation, and gradual fine-tuning that significantly improve the results of the FLORES-101 model.
\citet{sutawika2021data} use a standard Seq2Seq Transformer model without any training or architecture tricks, relying mainly on the strength of the data preprocessing techniques and filtering.

% Because of the limited computational resources and the setup with a very limited amount of data that we are considering, we cannot compare our approach to the above-mentioned methods.

Given our focus on a setup with very limited data and our available computational resources, we concentrate on evaluating our specific approach. Therefore, our results cannot be compared to the above-mentioned methods.

\subsubsection{LRMT for Finno-Ugric languages}

Regarding the Finno-Ugric languages, very few approaches are developed or tested on the benchmark. \citet{DBLP:journals/bjmc/TarsTF22} uses the standard M2M100 model \cite{m2m100} enhanced with the following steps: vocabulary extension in the tokenizer, data filtering, and preprocessing. \citet{yankovskaya-etal-2023-machine} improves previous results with back-translation and synthetic data as well as with the sampled high-resource language pairs to reduce catastrophic forgetting. Our models involve the same baselines; however, our training data consists of Sami languages (input) and Finish (output). Therefore, we also cannot compare the results directly to the above-mentioned methods.

\subsection{Personalized Federated Learning}

Federated Learning (FL) \citep{konecny2016federated, mcmahan2017communication} is a modern and rapidly developing part of Machine Learning, considering the training on the data distributed over multiple clients \citep{kairouz2021advances}. In the standard scenario, the goal is to train one global model that suits multiple clients, i.e., solve standard empirical risk minimization. In scenarios with heterogeneous data, the global model can show suboptimal results for particular clients, which necessitates considering Personalized Federated Learning (PFL) formulations to achieve better results on the client's data while getting benefits from collaboration with others. 

In the training of LLMs for the target (low-resource) language using the data in multiple languages, the goal is quite similar: to achieve good results for the target language while getting benefits from the model updates for other available languages. 
Informally speaking, by associating languages with clients, one can get a correspondence between PFL formulations and NLP formulations for low-resource languages. 
Therefore, in our work, we adjust the algorithmic ideas %from PFL literature (in particular, 
from \citep{tupitsa2024federated} %;
% see Section~\ref{section:methodology} for the details) 
to the training of LLMs for low-resource languages. We specify again that we do not use Personalized FL directly: our method is based on an analogy with the \algname{MeritFed} method from Federated Learning.

There also exist multiple PFL formulations and methods for solving them 
%(different from the ones considered/proposed by \citet{tupitsa2024federated}) 
with their own advantages and limitations, e.g., see \citep{fallah2020personalized, collins2021exploiting, hanzely2020lower, kulkarni2020survey, wu2021fast}. However, the works on PFL focus on different scenarios from our setup, i.e., they consider distributed training. %and address related challenges, e.g.,  communication efficiency, privacy, and security.

\subsection{Impact of Auxilary Data}

Many existing papers rely on 
auxiliary data, especially when the given dataset is too small. \citet{schroder-biemann-2020-estimating} automatically assesses the similarity of sequence tagging datasets to identify beneficial auxiliary data for Multi-Task Learning or Transfer Learning setups.
\citet{DBLP:conf/icml/ChenWGL022} propose a joint task and data scheduling model for auxiliary learning by creating a mapping from task, feature, and label information to the schedule in a parameter-efficient way.

Regarding LRMT, studies use the related languages when little data for the target language is given. One of the attempts to approach each language differently during training is made by \citet{huo-etal-2024-gradient-consistency}. They dynamically allocate parameters of an appropriate scale to each language direction based on the consistency between the gradient of the individual language and the average gradient. \citet{millour-etal-2024-agettivu-aggitivu,da-dalt-etal-2024-flor-effectiveness} %discuss whether it is possible to take advantage of resources and models available for closely related languages when no target language resources are available. They 
show that datasets on closely related languages are highly beneficial for applying to the target low-resource language. 
\citet{imanigooghari-etal-2023-glot500} also investigate the positive effects of closely related languages on the Glot-500 model. They analyze the impact of related languages via continued pre-training and confirm better performance for languages with their language family or script present in training.

% \eduard{Plan: shortly about Federated Learning and its various aspects, then say about global and personalized/local models, different formulations --> multiple approaches exist there. make sure you motivate it and explain why the federated nature is useful. make clear this is usually used for distributed computation settings, but not here. also clarify what personalization usuually means and what it means here.}

\section{Methodology}\label{section:methodology}

\begin{algorithm*}[t]
    \caption{\algname{MeritOpt}: General Algorithmic Framework for Learning from Heterogeneous Data}\label{alg:meritfed}
    \begin{algorithmic}[1] % 1 for numbering; 0 for not numbering
        \State {\bfseries Input:} Number of steps $T$, starting point $\x^0 \in \R^d$, stepsizes $\{\gamma_t\}_{t=1}^T$ ($\gamma_t > 0$), optimization update rule $\texttt{OptStep}(x, g, \gamma): \R^d \times \R^d \times \R \to \R^d$, datasets $\{\mD_{i}\}_{i=1}^n$, target validation dataset $\widehat\mD$
        \For{$t=0,1,\ldots, T$}
            \ForAll{$i= 1,\ldots,\n$ \textbf{in parallel}}
                \State Compute stochastic gradient 
                $\g_i(\x^t)$ from dataset $\mD_i$ 
            \EndFor            
        \State
        \label{lst:line:aux_problem} 
        $\w^{t+1} \approx \argmin\limits_{w\in \Delta_1^n} f_{\widehat\mD}\left(\texttt{OptStep}\left(x^t, \sum\limits_{i=1}^n w_i g_i(x^t), \gamma_t\right)\right)$
        % $\w^{t+1} \approx \argmin\limits_{w\in \Delta_1^n}f\left(x^t - \gamma \sum\limits_{i=1}^n \w_i  \g_i(\x^t)\right)$
        % \label{lst:line:algline7}
        \State $\x^{t+1} = \texttt{OptStep}\left(x^t, \sum\limits_{i=1}^n w_i^{t+1} g_i(x^t), \gamma_t\right)$ %\Comment{{\tiny Example: $\texttt{OptStep}\left(x^t, \sum\limits_{i=1}^n w_i^{t+1} g_i(x^t), \gamma_t\right) = x^t - \gamma_t\sum\limits_{i=1}^n w_i^{t+1} g_i(x^t)$}}
        \EndFor
    \end{algorithmic}
\end{algorithm*}

% \begin{algorithm}[ht]
%     \caption{\Algn: Merit-based Federated Learning for Diverse Datasets}\label{alg:meritfed}
%     \begin{algorithmic}[1] % 1 for numbering; 0 for not numbering
%         \State {\bfseries Input:} Starting point $\x^0 \in \R^d$, stepsize $\gamma > 0$
%         \For{$t=0,...$}
%             \State server sends $\x^t$ to each worker
%             \ForAll{\textbf{workers} $i= 1,\dots,\n$ \textbf{in parallel}}
%                 \State compute stochastic gradient 
%                 $\g_i(\x^t)$  from local data and \textbf{send} $\g_i(\x^t)$ to the server 
%             \EndFor            
%         \State
%         \label{lst:line:aux_problem} 
%         $\w^{t+1} \approx \argmin\limits_{w\in \Delta_1^n}f\left(x^{t+1}(\w)\right)$
%         % $\w^{t+1} \approx \argmin\limits_{w\in \Delta_1^n}f\left(x^t - \gamma \sum\limits_{i=1}^n \w_i  \g_i(\x^t)\right)$
%         % \label{lst:line:algline7}
%         \State $\x^{t+1} = \x^t - \gamma \sum\limits_{i=1}^n \w^{t+1}_i  \g_i(\x^t)$. 
%         \EndFor
%     \end{algorithmic}
% \end{algorithm}

\paragraph{General setup.} We start with the description of the general problem formulation that our approach is suitable for. That is, we consider the scenario when $n \geq 1$ datasets $\{\mD_i\}_{i=1}^n$ are available for training, and the goal is to train the model for some data distribution $\cD$ using this collection of datasets. More precisely, we focus on the standard learning problem \citep{shalev2014understanding}: $\min_{x\in \R^d} f_{\cD}(x)$, 
% \begin{equation}
% \textstyle
%     \min\limits_{x\in \R^d} f_{\cD}(x), \label{eq:risk_minimization_formulation}
% \end{equation}
where $f_{\cD}:\R^d \to \R$ is the expected loss computed for the data distribution $\cD$, i.e., $f_\cD := \E_{\xi\sim \cD}[f_{\xi}(x)]$ with $f_\xi : \R^d \to \R$ being a loss on sample $\xi$ and $\E_{\xi\sim \cD}[\cdot]$ denoting an expectation w.r.t.\ $\xi$ coming from the target distribution $\cD$, and $x \in \R^d$ represents a vector of model parameters, i.e., weights of the network. In practice, data distribution $\cD$ is typically unknown. Therefore, to approximate $f_{\cD}(x)$, finite dataset $\widehat \mD$ sampled from distribution $\cD$ is used. Throughout the paper, we call this dataset the target one and denote the corresponding (empirical) loss as $f_{\widehat \mD}(x)$. In addition, we assume that a collection of datasets $\{\mD_i\}_{i=1}^n$ is available for the training. 

We assume that $\mD_1$ is sampled from the target distribution $\cD$, and we make no assumptions on the other datasets. In particular, $\{\mD_i\}_{i=2}^n$ can be arbitrary heterogeneous and different from $\mD_1$ and $\widehat \mD$. However, if some of the available datasets are sampled from distributions that are close to $\cD$, they can be quite useful for the training. This idea serves as the main motivation behind our approach.

\paragraph{Algorithmic framework.} To solve the described problem, we propose a generic algorithmic framework -- \algname{MeritOpt} (see Algorithm~\ref{alg:meritfed}) -- inspired by \algname{MeritFed} proposed by \citet{tupitsa2024federated} for solving Personalized Federated Learning problems. \algname{MeritOpt} can be seen as a ``wrapper'' for an optimization method having update rule $x^{t+1} = \texttt{OptStep}(x^t, g(x^t), \gamma_t)$, where $x^t$ represents the weights of the model after step $t$, $g(x^t)$ is the stochastic (mini-batched) gradient computed at $x^t$, and $\gamma_t$ is the learning rate. For example, when the underlying method is Stochastic Gradient Descent (\algname{SGD}) \citep{robbins1951stochastic}, we have  $\texttt{OptStep}(x^t, g(x^t), \gamma_t) =  x^t - \gamma_t g(x^t)$ and Algorithm~\ref{alg:meritfed} reduces to \algname{MeritFed}\footnote{\algname{MeritFed} = \algname{MeritOpt-SGD}.} from \citep{tupitsa2024federated}. However, we can apply \algname{MeritOpt} to the update rule of any stochastic first-order method, e.g., \algname{Adam} \citep{kingma2014adam} and its variations, \algname{AdaGrad} \citep{streeter2010less, duchi2011adaptive}, \algname{RMSProp} \citep{hinton2012neural}, and other methods. In our experiments, we use \algname{Adam} as $\texttt{OptStep}(x, g, \gamma)$. The resulting method -- \algname{MeritOpt-Adam} -- is a new method that was never used or analyzed before.

In addition to the update rule $\texttt{OptStep}(x, g, \gamma)$, \algname{MeritOpt} takes $n$ input datasets $\{\mD_i\}_{i=1}^n$ and $1$ target validation dataset $\widehat \mD$. At each iteration, the method computes (mini-batched) stochastic gradient $g_i(x^t)$ using the corresponding dataset $\mD_i$ for each $i=1,\ldots,n$. Then, to construct the update direction, \algname{MeritOpt} searches appropriate aggregation weights (that can be interpreted as ``merits'' for different languages) $w^{t+1} = (w_1^{t+1},\ldots, w_n^{t+1})^\top$ (see Line~\ref{lst:line:aux_problem}) and then makes a step $\x^{t+1} = \texttt{OptStep}\left(x^t, \sum_{i=1}^n w_i^{t+1} g_i(x^t), \gamma_t\right)$ using the computed weighted average of the stochastic gradients. We emphasize that the choice of aggregation weights $w^{t+1}$ is crucial: for example, if datasets $\{\mD_i\}_{i=2}^n$ came from distributions significantly different from the target distribution $\cD$ and we choose uniform weights, i.e., $w_1^{t+1} = \ldots = w_n^{t+1} = \nicefrac{1}{n}$, then the optimization step with the update vector $\sum_{i=1}^n w_i^{t+1} g_i(x^t)$ can be useless (on average) in terms of solving the target problem. Moreover, if some datasets came from distributions close to $\cD$, it is natural to use the corresponding stochastic gradients with larger weights to benefit from them.

\algname{MeritOpt} addresses this issue in Line~\ref{lst:line:aux_problem}: the goal is to find aggregation weights $w^{t+1} \in \Delta_1^n$, where $\Delta_1^n := \{y \in \R^n\mid \sum_{i=1}^n y_i = 1,\; y_i \geq 0 \; \forall \; i=1,\ldots,n\}$ is the $n$-dimensional probability simplex, such that the loss $f_{\widehat \mD}$ on the target validation dataset $\widehat \mD$ is minimized after the step $\texttt{OptStep}\left(x^t, \sum_{i=1}^n w_i^{t+1} g_i(x^t), \gamma_t\right)$ that depends on $w^{t+1}$. If $\widehat \mD$ is sufficiently large, then $f_{\widehat \mD}$ can be seen as a good approximation of $f_{\cD}$ 
\citep{shalev2009stochastic},
% \citep{shalev2009stochastic, liu2024new}, 
and 
optimizing $f_{\widehat\mD}$ leads to sufficiently good solution for $f_{\cD}$. In other words, given stochastic gradients $g_i(x^t)$ computed from different datasets $\{\mD_i\}_{i=1}^n$, \algname{MeritOpt} tries to find the best-weighted average of them to make an optimization step. Following \citet{tupitsa2024federated}, we apply several steps of Stochastic Mirror Descent  \citep{nemirovskij1983problem} to solve the problem in Line~\ref{lst:line:aux_problem} approximately (see Appendix~\ref{appendix:MD}).

\begin{table*}[ht!]
\centering
\small
\resizebox{\textwidth}{!}{
\begin{tabular}{lcccccccc}
\toprule
\multirow{2}{*}{\textbf{Method}} & \multicolumn{2}{c}{\textbf{Inari Sami}} & \multicolumn{2}{c}{\textbf{Skolt Sami}} & \multicolumn{2}{c}{\textbf{South Sami}} & \multicolumn{2}{c}{\textbf{North Sami}} \\
\cmidrule(lr){2-3} \cmidrule(lr){4-5} \cmidrule(lr){6-7} \cmidrule(lr){8-9}

 & \textbf{Score} & \textbf{Steps} & \textbf{Score} & \textbf{Steps} & \textbf{Score} & \textbf{Steps} & \textbf{Score} & \textbf{Steps} \\
\midrule
FT\textsubscript{OnlyT} & 9.44 $\pm$ 0.20 & 1.5K & 38.83 $\pm$ 0.31 & 2K & 48.70 $\pm$ 0.14 & 8K & 39.26 $\pm$ 0.33 & 53K \\
FT\textsubscript{All} & 5.56 $\pm$ 0.29 & 21K & 34.11 $\pm$ 0.23 & 23K & 44.62 $\pm$ 0.10 & 23K & 33.57 $\pm$ 2.34 & 12K \\
FT\textsubscript{NoT} & 2.38 $\pm$ 0.09 & 16K & 11.62 $\pm$ 0.37 & 23K & 16.63 $\pm$ 0.35 & 16K & 10.16 $\pm$ 0.16 & 2K \\
\midrule
CP\textsubscript{All} & 51.39 $\pm$ 0.05 & 30K & 44.90 $\pm$ 0.12 & 25K & 11.60 $\pm$ 0.29 & 23K & 39.78 $\pm$ 0.08 & 69K \\
CP\textsubscript{NoT} & 50.14 $\pm$ 0.04 & 31K & 43.40 $\pm$ 0.13 & 25K & 11.09 $\pm$ 0.24 & 23K & 39.30 $\pm$ 0.18 & 65K \\
\midrule
\algname{MeritOpt} & \textbf{52.08 $\pm$ 0.01} & 12K & \textbf{50.27 $\pm$ 0.17} & 12K & \textbf{13.26 $\pm$ 0.17} & 2.5K & 38.526 $\pm$ 1.39 & 30K \\
\bottomrule
\end{tabular}
}
\caption{Mean SpBLEU scores and the number of steps required to reach them for baselines and \algname{MeritOpt} within Finno-Samic low-resource languages.}
\label{table:results_fin}
\vspace{0.3cm}
\centering
\small
\resizebox{\textwidth}{!}{
\begin{tabular}{lcccccccccccc}
\toprule
\multirow{2}{*}{\textbf{Method}} & \multicolumn{6}{c}{\textbf{Tagalog}} & \multicolumn{6}{c}{\textbf{Java}} \\
\cmidrule(lr){2-7} \cmidrule(lr){8-13}
 & \multicolumn{2}{c}{\textbf{79K}} & \multicolumn{2}{c}{\textbf{155K}} & \multicolumn{2}{c}{\textbf{555K}} & \multicolumn{2}{c}{\textbf{79K}} & \multicolumn{2}{c}{\textbf{128K}} & \multicolumn{2}{c}{\textbf{555K}} \\
\cmidrule(lr){2-3} \cmidrule(lr){4-5} \cmidrule(lr){6-7} \cmidrule(lr){8-9} \cmidrule(lr){10-11} \cmidrule(lr){12-13}
 & \textbf{Score} & \textbf{Steps} & \textbf{Score} & \textbf{Steps} & \textbf{Score} & \textbf{Steps} & \textbf{Score} & \textbf{Steps} & \textbf{Score} & \textbf{Steps} & \textbf{Score} & \textbf{Steps} \\
\midrule
FT\textsubscript{OnlyT} & 28.69 $\pm$ 0.10 & 8K & 30.48 $\pm$ 0.05 & 44K & 33.88 $\pm$ 0.07 & 52K & 19.23 $\pm$ 0.02 & 500 & 19.69 $\pm$ 0.01 & 1K & 20.75 $\pm$ 0.10 & 3.5K \\
FT\textsubscript{All} & 24.78 $\pm$ 0.02 & 12K & 26.53 $\pm$ 0.17 & 25K & 30.02 $\pm$ 0.03 & 79K & 19.26 $\pm$ 0.02 & 12K & 19.28 $\pm$ 0.07 & 25K & 19.92 $\pm$ 0.06 & 85K \\
FT\textsubscript{NoT} & 20.45 $\pm$ 0.08 & 7K & 20.41 $\pm$ 0.07 & 11K & 20.34 $\pm$ 0.09 & 53K & 18.73 $\pm$ 0.02 & 12K & 18.80 $\pm$ 0.04 & 25K & 18.94 $\pm$ 0.09 & 85K \\
\midrule
CP\textsubscript{All} & 29.24 $\pm$ 0.06 & 21K & 30.99 $\pm$ 0.04 & 40K & 33.89 $\pm$ 0.15 & 124K & 19.43 $\pm$ 0.14 & 12K & 20.05 $\pm$ 0.12 & 25K & 20.97 $\pm$ 0.13 & 87K \\
CP\textsubscript{NoT} & 28.72 $\pm$ 0.16 & 15K & 30.50 $\pm$ 0.12 & 42K & 33.74 $\pm$ 0.19 & 129K & 19.46 $\pm$ 0.12 & 12K & 19.95 $\pm$ 0.12 & 25K & 21.19 $\pm$ 0.09 & 89K \\
\midrule
\algname{MeritOpt} & \textbf{29.73 $\pm$ 0.04} & 14K & \textbf{31.42 $\pm$ 0.07} & 14K & 33.53 $\pm$ 0.27 & 47K & \textbf{19.74 $\pm$ 0.03} & 2K & \textbf{20.23 $\pm$ 0.11} & 3K & \textbf{21.44 $\pm$ 0.13} & 8K \\
\bottomrule
\end{tabular}
}
\caption{Mean SpBLEU scores and the number of steps required to reach them for baselines and \algname{MeritOpt} within the different data sizes of Javanese and Tagalog languages.}
\label{table:results_ind}
\end{table*}

\paragraph{Application to NLP.} The described approach can be applied to the training of LLMs for LRLs. In this case, $\{\mD_i\}_{i=1}^n$ correspond to the input datasets in $n$ different languages. In particular, $\mD_1$ is the training dataset for the target language\footnote{One can interpret all possible texts in the target language as some distribution $\cD$. In this interpretation, $\mD_1$ can be seen as some dataset sampled from language $\cD$.} and $\widehat\mD$ is the target validation dataset for the same language. The remaining datasets $\{\mD_i\}_{i=2}^n$ are for other languages. Some of these languages can be related to the target one, but, in general, we allow the usage of datasets in significantly different languages as well: \algname{MeritOpt} automatically adjusts aggregation weights and assigns higher weights to more beneficial languages. 
% in terms of the training the model for the target language. 
Therefore, aggregation weights $w^{t+1}$ can be used to measure the impact of selected languages on the model's training for the target language. In other words, we extend the training target language dataset and prevent drifting towards the solution for other languages. We also note that in the original work \citep{tupitsa2024federated}, \algname{MeritFed} was tested on on different problems (image and emotion classification with different models).

\section{Experiments}

In this section, we apply the methodology to learn low-resource languages with the help of related languages. We also discuss the data used, the baselines, and the evaluation metrics.

\subsection{Datasets}

To test the developed method, we consider datasets with related languages that either belong to the same language family or are geographically related, which we expect to be ``helpful'' during the training procedure. We focus exclusively on settings with related languages, as this approach is more computationally efficient. However, our method does not have any inherent bias towards language relatedness and can be applied to any number of languages in the training set. As demonstrated in Section~\ref{sec:results}, it becomes even more effective with addition of unrelated languages.

For our experiments, we select a subset from the Large-Scale Multilingual Machine Translation Shared Task (Small Track \#2) \cite{wenzek-etal-2021-findings} and the subset of Sami languages from the multilingual benchmark for Finno-Ugric languages \cite{yankovskaya-etal-2023-machine}. We describe each dataset in detail in the following paragraphs.

\paragraph{South East Asian languages Dataset.}

For the first round of experiments, we select one of the small tracks, Large-Scale Multilingual Machine Translation Shared Task, comprising translation pairs between fairly related languages and English and not requiring substantial computational resources at training time. We stick to Javanese, Indonesian, Malay, Tagalog, and Tamil as input languages and English as output. As target languages, we utilize Javanese and Tagalog as the smallest language pairs in the dataset. We perform our experiments on multiple dataset scales: 80K (small), 150K (medium), and 500K (large). 
Our primary goal is to test the method; therefore, we do not perform experiments on the whole dataset, leaving this to future work.  
For additional experiments, we utilize the Hungarian dataset from Small Track \#1. All the dataset statistics are provided in Table \ref{tab:stats_ind} for the initial dataset and for the datasets created for our experiments.

\begin{figure*}[th!]
    \centering
    \includegraphics[width=\textwidth]{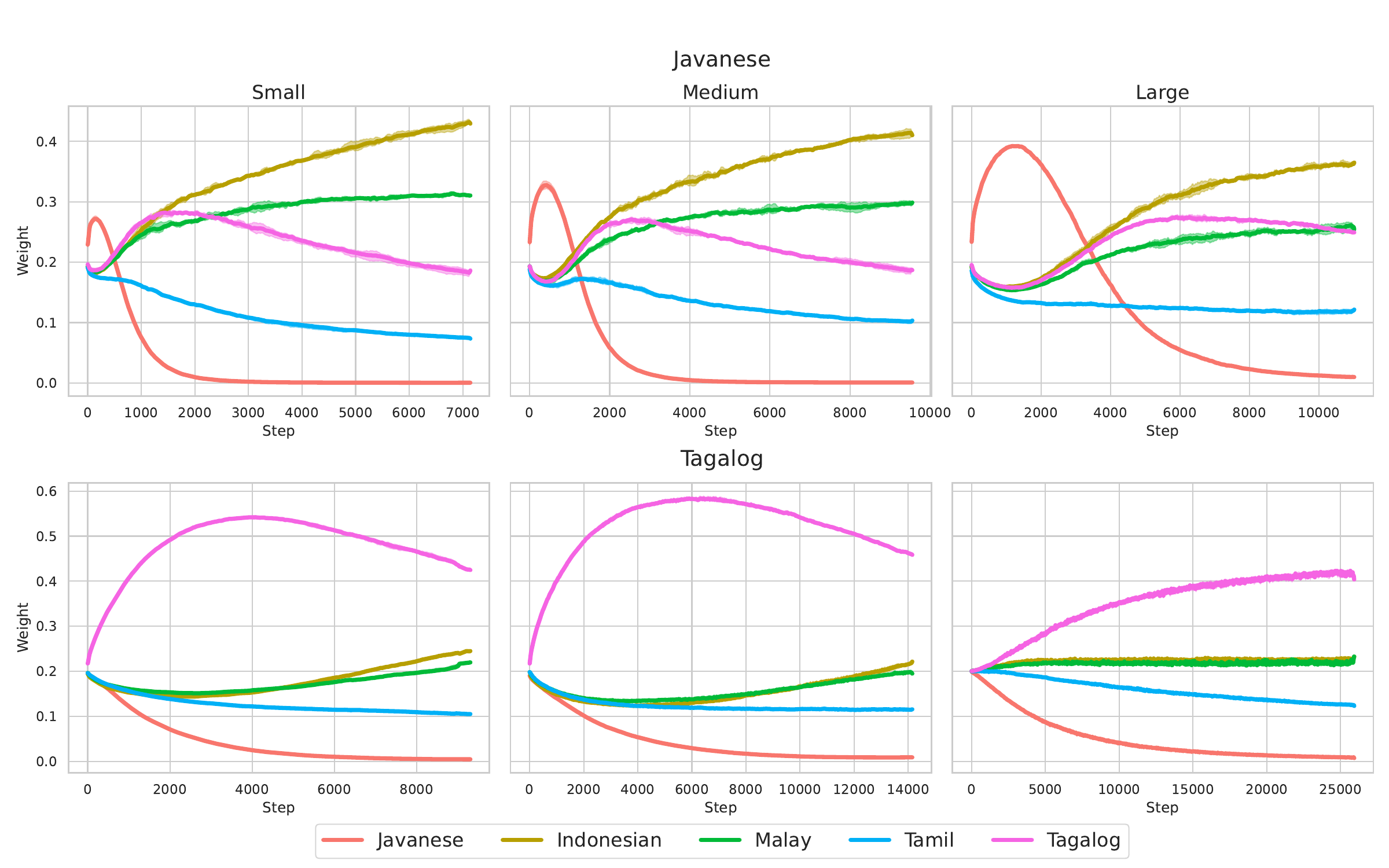}
    \caption{Weights distribution for South East Asian languages. Target languages and data sizes are in captions.}
    \label{fig:weight_ind}
\end{figure*}

\paragraph{Finno-Samic Languages Dataset.}
Regarding the dataset compiled from the Finno-Ugric benchmark \cite{yankovskaya-etal-2023-machine}, we stick to the Sami languages as the only option matching our criteria: parallel training datasets of different sizes with the same output language (Finnish) for those pairs, parallel development and test datasets of good quality. Unfortunately, such data is available only for Finno-Samic languages\footnote{\url{https://huggingface.co/datasets/tartuNLP/finno-ugric-benchmark},\\ \url{https://huggingface.co/datasets/tartuNLP/finno-ugric-train}}, such as tartuNLP/finno-ugric-benchmark North Sami, South Sami, Inari Sami, Skolt Sami. The dataset statistics are presented in Table \ref{tab:stats_sami}. In future experiments, we plan to extend the datasets to other languages and directions from the benchmark.

\subsection{Baselines}

For our baselines, we consider fine-tuning to the target language both with and without various forms of prior continual pretraining:
\begin{itemize}
    \item FT\textsubscript{All} --- Fine-tuning to all languages including the target language;
    \item FT\textsubscript{NoT} --- Fine-tuning to all languages except the target language;
    \item FT\textsubscript{OnlyT} --- Fine-tune to the target language only;
    \item CP\textsubscript{All} --- Continuous Pretraining to all languages, followed by additional fine-tuning to the target language;
    \item CP\textsubscript{NoT} --- Continuous Pretraining to all languages but the target, followed by additional fine-tuning to the target language.
\end{itemize}
We use the M2M100 model with 418M parameters as our base model \cite{fan2020englishcentric}. For Finno-Ugric languages, special language tokens are added and learned since the model was not pretrained for those languages. More training details and configurations are provided in Appendix \ref{sec:appendix_params}.

\subsection{Evaluation}

We use SpBLEU metrics in our evaluation as in \citet{sutawika2021data}, utilizing SacreBLEU \cite{post2018call}. The generation parameters are adopted from \citet{xie-etal-2021-tentrans}, employing beam search with 4 beams, and the temperature set to 1.

\begin{figure*}[ht!]
    \centering
    \includegraphics[width=\textwidth]{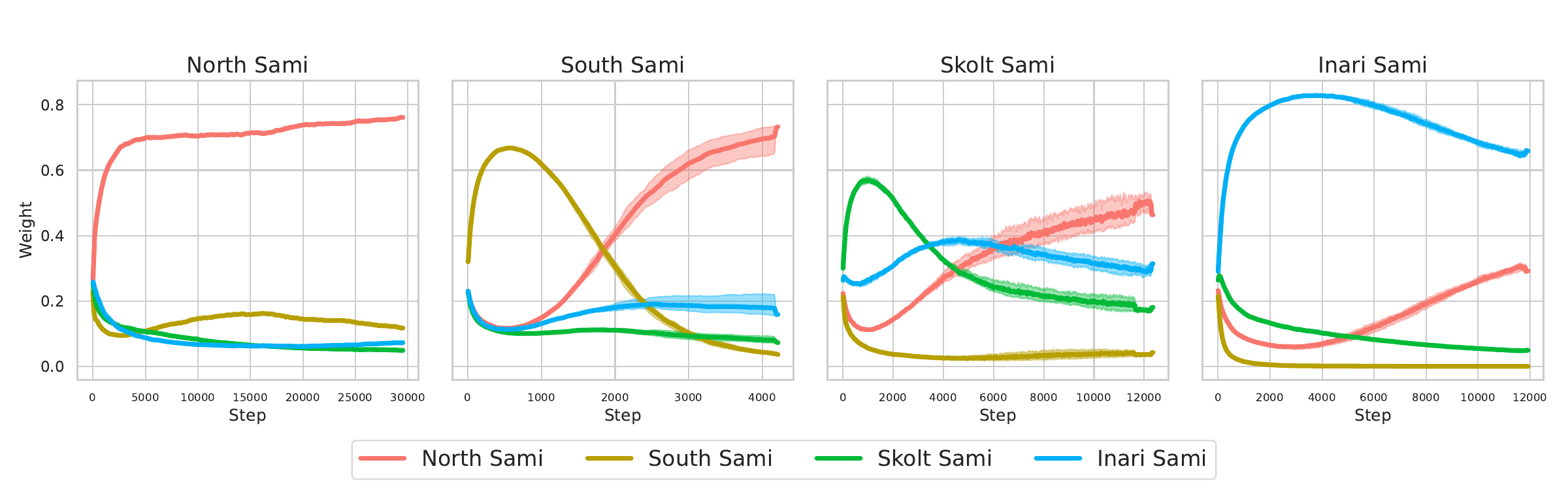}
     \vspace{-0.8cm}
    \caption{Weights distribution across Finno-Samic languages. Target languages are mentioned in captions.}
    \label{fig:weight_fin}
    \vspace{-0.2cm}
\end{figure*}

\section{Results and Discussion} \label{sec:results}

%In this section, we discuss the results from both datasets and language weight distributions during training. We also conduct several ablation studies on dataset size, non-related languages, and Mirror Descent hyperparameters. Finally, we provide a brief description of the theoretical results.

Tables \ref{table:results_fin} and \ref{table:results_ind} show that \algname{MeritOpt} is indeed helpful during training: our approach achieves better performance for most setups and languages. for Javanese and Tagalog languages (small and medium) and for Sami languages of comparable sizes (South, Scolt, and Inari). 
%We report scores for approaches without Continuous Pretraining in Appendix \ref{sec:appendix_res}, as they are always worse than with CP. %The lowest results are demonstrated when training the system on the target language input only. %Simultaneous fine-tuning on the related languages only and with the target languages achieves higher results. However, they still lag behind the continuous pre-training baselines (with and without the target language) for all languages.

\begin{figure}[t!]
    \centering
    \vspace{-0.2cm}
    \includegraphics[width=0.4\textwidth]{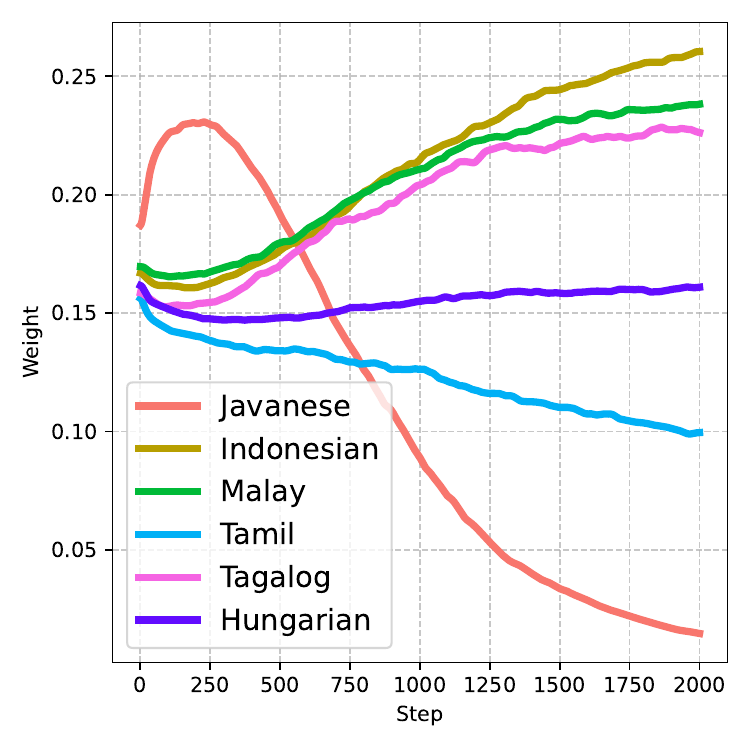}
    \vspace{-0.3cm}
    \caption{Weights distribution for target Indonesian language with unrelated Hungarian included.}
    \label{fig:hun_weig}
    \vspace{-0.2cm}
\end{figure}

\paragraph{Impact of Aggregation Weights.}\label{sec:unrelated_lang} We can see that the methods assign higher weights to the target language at first, followed by a drop, while other weights increase. Therefore, Javanese benefits more from the Indonesian language, while Tagalog's, higher-weighted languages are Indonesian and Malay. Interestingly, while spoken in South East Asia, the Tamil language does not belong to the same language family as the others. This fact is reflected in Figure~\ref{fig:weight_ind}: Tamil always contributes less than other languages. For Sami languages, North Sami seems always to be the most beneficial.

\paragraph{No Overfitting.} An important observation is that the algorithm helps to prevent the model from overfitting: the weight of the target language decreases once the model learns the small amount of data available for the target language; additional languages serve as regularization to keep the model converging. Probably, that partially explains the non-zero weights of Tamil, which does not belong to the same language family, although being spoken in South East Asia. 

\paragraph{Unrelated Language.} To check the hypothesis that unrelated language serves as regularization, we conducted an additional experiment and added the Hungarian language from the Finno-Ugric family to training. As shown in Figure~\ref{fig:hun_weig}, its weights are also non-zero. 
Moreover, the SpBLEU scores remained nearly consistent across all MD parameters and Adaptive Batch configurations, supporting the regularization role of additional languages. %We explain this with the regularisation role of additional languages, including non-related Hungarian.
To further extend our experiment and validate our hypothesis, we evaluated the model's performance on the Java small dataset by incorporating five unrelated languages (Croatian, Serbian, Macedonian, Estonian, Hungarian) into the training set. The results in Table~\ref{table:5unrel} demonstrate that the model benefits from the inclusion of additional languages, even being unrelated.

\paragraph{Size of the Target Language Dataset.}

For Tagalog-large and North Sami, the algorithm relies on the target language dataset more than on additional languages and does not outperform the Continuous Pretraining baseline. On the contrary, for small and medium datasets, the algorithm needs from 2 to 10 times fewer main gradient steps to outperform the baselines. 

We assume that this happens because the amount of data from the target language is enough, and the algorithm keeps assigning high weights to the target language and trains the model on the target language only.  
Another possible reason for the inferior performance could be excessive gradient steps involving non-target languages. This might ``distract'' the model and fail to provide significant benefits. Since at each step, we compute the stochastic gradients for other languages, too, the method does not pass the whole dataset of the target language, given the computational resources for the experiment. Therefore, the method does not utilize all potentially useful information from the target language. 

\begin{figure}[t!]
    \centering
    \vspace{-0.3cm}
    \includegraphics[width=0.4\textwidth]{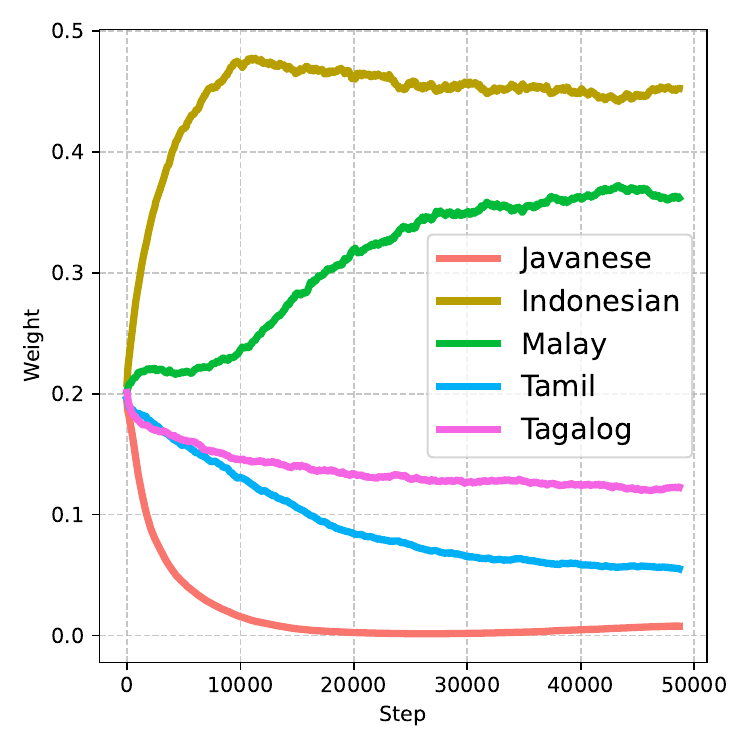}
     \vspace{-0.3cm}
    \caption{Weights distribution for languages with target Indonesian on \textit{small} subset.}
    \label{fig:ind_weig}
    \vspace{-0.2cm}
\end{figure}

This hypothesis is supported by an additional experiment on the Indonesian language as the target language with the biggest dataset to see the distribution of the weights for a longer number of steps ($\sim50$K). From Figure~\ref{fig:ind_weig}, we observe the evolution of corresponding aggregation weights: it keeps growing during the training, which indicates its significantly higher importance on the model quality than other languages.
Once the model learns the dataset better, the weight of the Indonesian slightly decreases and gets stuck, while the weight of the Malay starts to grow. %Therefore, we can conclude that the dataset size of the target language has a high influence on the process of model training and weight distribution. 
We assume that these observations might be useful for further experiments: languages that stop contributing to the algorithm convergence can be ``dropped'' during training.

\paragraph{Adaptive Batch Experiments.} \label{sec:adaptive_batch}

We hypothesize that leveraging high-resource languages could improve gradient approximation by providing more samples. Based on this, we develop an Adaptive Batch procedure. This method allocates the total batch size ($512$ in our experiments) and samples batch size from the total size for each language proportionally to the percentage of each language present in the dataset. Thus, high-resource languages receive larger batch sizes. To optimize convergence, we set batch size limits, with a lower bound of $32$ and an upper bound of $128$, as shown to be effective in previous studies \cite{keskar2017largebatch,bengio2012practical}.

\begin{figure}[t!]
    \centering
    \includegraphics[width=0.4\textwidth]{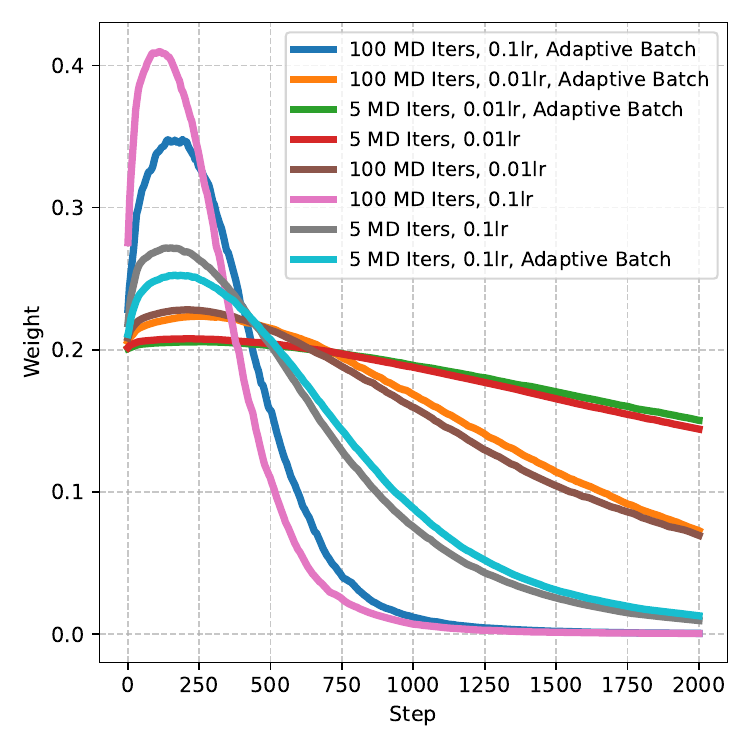}
    \vspace{-0.3cm}
    \caption{Weights for target language (Javanese-small) with different Mirror Descent parameters.}
    \label{fig:md_params}
    \vspace{-0.2cm}
\end{figure}

However, our results indicate that the Adaptive Batch procedure is rarely beneficial. We believe this is due to the downside of better gradient approximation. Our method suggests that assigning higher weight to high-resource languages due to their well-estimated gradients may hinder the learning of the target language. This is illustrated in Figure~\ref{fig:md_params}, where adding an Adaptive Batch leads to a lower weight for the target language.

\paragraph{Mirror Descent Parameters Impact.} \label{sec:MD_params}

We conduct experiments with various \algname{MeritOpt} settings, adjusting the Mirror Descent learning rate to $0.1$ and $0.01$ and the number of iterations to $5$ or $100$. These experiments are performed on the \textit{small} subset of the South East Asian dataset, using Javanese as the target language. 

As shown in Table~\ref{table:ablation}, the Mirror Descent parameters have little impact, with no clear trend emerging\footnote{We conjecture that (Stochastic) Mirror Descent struggles to solve the auxiliary problem from Line~\ref{lst:line:aux_problem} with sufficiently good accuracy since it can be viewed as a variant of \algname{SGD} for problems with non-Euclidean prox-structure and \algname{SGD} is known to perform poorly in NLP tasks \citep{zhang2020adaptive}. Investigation of other alternatives (e.g., MD version of \algname{Adam}) is a prominent direction for future research.}. For 5 iterations, a higher learning rate, and no Adaptive Batch, the algorithm performs better when no unrelated language is present. However, this changes when an unrelated language is included. For 100 iterations, a lower learning rate, and no Adaptive Batch, the model consistently yields better results. Based on those observations, we have chosen 5 MD iterations with a learning rate of $0.1$ for all experiments due to its high performance and faster computation times.

% \eduard{Add conclusion of why MD parameters impact target lang weight the way they do in figure below}

% \subsubsection{Addition of an Unrelated Language} \label{sec:unrelated_lang}
% We also experimented with adding Hungarian to the Indonesian group to see if an unrelated language would impact performance. Using Javanese as the target language with a total of 79K data samples, we observed the results presented in Table~\ref{table:ablation}. The SpBLEU scores remained nearly consistent across all Mirror Descent parameters (discussed in Section~\ref{sec:MD_params}) and Adaptive Batch configuration (discussed in Section~\ref{sec:adaptive_batch}). In some cases, adding Hungarian even improved quality, potentially serving as a form of regularization or as an additional source of general language knowledge. This could be evidenced by inspecting the gradient weight for Hungarian at Figure~\ref{fig:hun_weig}, which does not drop or rise through the training process.

\begin{table}[t!]
\centering
\resizebox{0.48\textwidth}{!}{
\begin{tabular}{cccccc}
\toprule
\multirow{2}{*}{\shortstack{\textbf{MD} \\ \textbf{Iterations}}}& \multirow{2}{*}{\shortstack{\textbf{Learning} \\ \textbf{Rate}}} & \multirow{2}{*}{\shortstack{\textbf{Adaptive} \\ \textbf{Batch}}} & \multicolumn{2}{c}{\textbf{SpBLEU}} \\
\cmidrule(lr){4-5}
 & & & \textbf{Relevant} & \textbf{+Irrelevant} \\
\midrule
\multirow{4}{*}{5} & \multirow{2}{*}{0.1} & - & 19.74 & 19.79 \\
 & & + & 19.67 & 19.62 \\
\cmidrule(lr){2-5}
 & \multirow{2}{*}{0.01} & - & 19.72 & 19.72 \\
 & & + & 19.67 & 19.81 \\
\midrule
\multirow{4}{*}{100} & \multirow{2}{*}{0.1} & - & 19.70 & 19.65 \\
 & & + & 19.57 & 19.59 \\
\cmidrule(lr){2-5}
 & \multirow{2}{*}{0.01} & - & 19.75 & 19.72 \\
 & & + & 19.74 & 19.64 \\
\bottomrule
\end{tabular}
}
\caption{Scores and Settings Grouped by Mirror Descent iterations for the Javanese-small dataset.}
\label{table:ablation}
%\vspace{-0.3cm}
\end{table}

\begin{table}[t!]
\centering
\resizebox{0.4\textwidth}{!}{
\begin{tabular}{cc}
\toprule
\textbf{Languages} & \textbf{SpBLEU} \\
\midrule
South East Asian & 19.74 \\
South East Asian + Hungarian & 19.79 \\
South East Asian + 5 European & \textbf{19.98} \\
\bottomrule
\end{tabular}
}
\caption{Scores and languages in train for the Javanese-small dataset.}
\label{table:5unrel}
\vspace{-0.2cm}
\end{table}

\paragraph{Theoretical Results.} 

% Here, we informally summarize our main theoretical result for \algname{MeritOpt}.

% \begin{theorem}\label{thm:main_result_simplified}
%     Let the function $f_{\cD}$ be $L$-smooth. Assume that $\E_{\circledSix}[f_{\cD}\left(x^{t+1}\right)] - \min_{w\in \Delta_1^n} f_{\cD}\left(\texttt{OptStep}\left(x^t, \sum_{i=1}^n w_i g_i(x^t), \gamma_t\right)\right) \leq \delta$, where $\E_{\circledSix}[\cdot]$ is an expectation w.r.t.\ the randomness in solving Line~\ref{lst:line:aux_problem}. If the proof for the method $x^{t+1} = \texttt{OptStep}\left(x^t, g(x^t), \gamma_t\right)$ with $g(x^t)$ being an unbiased estimate of $\nabla f_{\cD}(x^t)$ with bounded variance is based on one-iteration analysis, then for sufficiently small $\gamma_t := \gamma > 0$ \algname{MeritOpt} converges to $\delta$-dependent neighborhood of the first-order stationary point of $f_{\cD}$.
% \end{theorem}

We prove that under certain assumptions on the underlying optimizer $\texttt{OptStep}$, \algname{MeritOpt} converges to the neighborhood of the solution of the target problem when (i) the learning rate is small enough, (ii) $\widehat\mD$ is sufficiently large such that $f_{\widehat \mD}$ is close to $f_{\cD}$, and (iii) the auxiliary problem in Line~\ref{lst:line:aux_problem} is solved with a good accuracy. In Appendix~\ref{appendix:theory}, we provide missing theoretical details (including the proofs) and show that \algname{SGD}, \algname{RMSProp}, \algname{AdaGrad-Norm} satisfy our assumptions. We emphasize that the theoretical results for \algname{MeritOpt-RMSProp} and \algname{MeritOpt-AdaGrad-Norm} are new, since \citet{tupitsa2024federated} provide the theoretical convergence analysis only for the \algname{MeritFed} (\algname{MeritOpt-SGD}) version.

\section{Conclusion}

In this paper, we implement the \algname{MeritOpt} algorithm from the Personalised Federated Learning to the Low-Resource Machine Translation task. We show that it can achieve better results than traditional approaches and requires 2 to 10 times fewer gradient steps than baselines (e.g., 8K vs. 85K, 12K vs. 23K). \algname{MeritOpt} also allows us to observe the weight distribution between the target and related languages: Javanese benefits more from the Indonesian language, while for Tagalog, the most important languages are Indonesian and Malay. Different weights for different languages also prevent the model from overfitting: after learning the target language dataset, its weights are dropped down while other weights start growing. Another takeaway is about the target dataset size: the bigger the dataset is, the more the algorithm keeps relying on it rather than on the auxiliary languages. This might result in worse model performance and ``distract’’ the model from convergence. %Our future solution would be to drop non-contributing languages during the training process.

\section*{Limitations}

\begin{itemize}
\item We report results only on Low-Resource MT, while a wide variety of NLP tasks are available. We leave further investigation of \algname{MeritOpt} to other NLP tasks for future work.
\item We report results only on M2M100, while numerous LLMs are available. An alternative model with the \algname{MeritOpt} algorithm applied could further improve the results. The research focuses on the algorithm application to the LRMT task and not on an exhaustive search of all LLM models. 
\item We limit our dataset in terms of size and language variety because of high computational costs and limited resources available.
\item  Our setup with the limited amount of languages and training data used is not designed to directly compare with the existing approaches.
\item We retain all languages during training, even those that do not contribute, which affects the efficiency of the training procedure.

\end{itemize}

\section*{Ethical Statement}
In our research, we utilize the M2M100 model, which has been pre-trained on a diverse MT corpus, including user-generated content. The datasets we use for additional model training have already been presented in WMT-21 Shared Task and Finno-Ugric Benchmark. Although we expect them to be filtered from harmful content, it is important to recognize that some biases may still persist in the model outputs. 

This acknowledgment does not undermine the validity of our methods. We have designed our techniques to be flexible, allowing them to be applied to alternative pre-trained models that have undergone more rigorous debiasing processes. To the best of our knowledge, aside from the challenge of mitigating inherent biases, our work does not raise any additional ethical concerns.

\section*{Acknowledgments}
The work of N.~Tupitsa has been financially supported by The Analytical Center for the Government of the Russian Federation (Agreement No. 70-2021-00143 01.11.2021, IGK 000000D730324P540002)

% This document has been adapted
% by Steven Bethard, Ryan Cotterell and Rui Yan
% from the instructions for earlier ACL and NAACL proceedings, including those for
% ACL 2019 by Douwe Kiela and Ivan Vuli\'{c},
% NAACL 2019 by Stephanie Lukin and Alla Roskovskaya,
% ACL 2018 by Shay Cohen, Kevin Gimpel, and Wei Lu,
% NAACL 2018 by Margaret Mitchell and Stephanie Lukin,
% Bib\TeX{} suggestions for (NA)ACL 2017/2018 from Jason Eisner,
% ACL 2017 by Dan Gildea and Min-Yen Kan,
% NAACL 2017 by Margaret Mitchell,
% ACL 2012 by Maggie Li and Michael White,
% ACL 2010 by Jing-Shin Chang and Philipp Koehn,
% ACL 2008 by Johanna D. Moore, Simone Teufel, James Allan, and Sadaoki Furui,
% ACL 2005 by Hwee Tou Ng and Kemal Oflazer,
% ACL 2002 by Eugene Charniak and Dekang Lin,
% and earlier ACL and EACL formats written by several people, including
% John Chen, Henry S. Thompson and Donald Walker.
% Additional elements were taken from the formatting instructions of the \emph{International Joint Conference on Artificial Intelligence} and the \emph{Conference on Computer Vision and Pattern Recognition}.

% Bibliography entries for the entire Anthology, followed by custom entries
%\bibliography{anthology,custom}
% Custom bibliography entries only
\bibliography{custom}

\begin{thebibliography}{54}
\providecommand{\natexlab}[1]{#1}

\bibitem[{Beck(2017)}]{beck2017first}
Amir Beck. 2017.
\newblock \emph{First-order methods in optimization}.
\newblock SIAM.

\bibitem[{Bengio(2012)}]{bengio2012practical}
Yoshua Bengio. 2012.
\newblock Practical recommendations for gradient-based training of deep architectures.
\newblock In \emph{Neural networks: Tricks of the trade: Second edition}, pages 437--478. Springer.

\bibitem[{Budiwati et~al.(2021)Budiwati, Fatyanosa, Data, Wijaya, Telnoni, Suryani, Pratondo, and Aritsugi}]{budiwati-etal-2021-optimize}
Sari~Dewi Budiwati, Tirana Fatyanosa, Mahendra Data, Dedy~Rahman Wijaya, Patrick~Adolf Telnoni, Arie~Ardiyanti Suryani, Agus Pratondo, and Masayoshi Aritsugi. 2021.
\newblock \href {https://aclanthology.org/2021.wmt-1.47} {To optimize, or not to optimize, that is the question: {T}el{U}-{KU} models for {WMT}21 large-scale multilingual machine translation}.
\newblock In \emph{Proceedings of the Sixth Conference on Machine Translation}, pages 387--397, Online. Association for Computational Linguistics.

\bibitem[{Chen et~al.(2022)Chen, Wang, Guan, Liu, and Zhu}]{DBLP:conf/icml/ChenWGL022}
Hong Chen, Xin Wang, Chaoyu Guan, Yue Liu, and Wenwu Zhu. 2022.
\newblock Auxiliary learning with joint task and data scheduling.
\newblock In \emph{{ICML}}, volume 162 of \emph{Proceedings of Machine Learning Research}, pages 3634--3647. {PMLR}.

\bibitem[{Collins et~al.(2021)Collins, Hassani, Mokhtari, and Shakkottai}]{collins2021exploiting}
Liam Collins, Hamed Hassani, Aryan Mokhtari, and Sanjay Shakkottai. 2021.
\newblock Exploiting shared representations for personalized federated learning.
\newblock In \emph{International conference on machine learning}, pages 2089--2099. PMLR.

\bibitem[{Da~Dalt et~al.(2024)Da~Dalt, Llop, Baucells, Pamies, Xu, Gonzalez-Agirre, and Villegas}]{da-dalt-etal-2024-flor-effectiveness}
Severino Da~Dalt, Joan Llop, Irene Baucells, Marc Pamies, Yishi Xu, Aitor Gonzalez-Agirre, and Marta Villegas. 2024.
\newblock \href {https://aclanthology.org/2024.lrec-main.650} {{FLOR}: On the effectiveness of language adaptation}.
\newblock In \emph{Proceedings of the 2024 Joint International Conference on Computational Linguistics, Language Resources and Evaluation (LREC-COLING 2024)}, pages 7377--7388, Torino, Italia. ELRA and ICCL.

\bibitem[{Duchi et~al.(2011)Duchi, Hazan, and Singer}]{duchi2011adaptive}
John Duchi, Elad Hazan, and Yoram Singer. 2011.
\newblock Adaptive subgradient methods for online learning and stochastic optimization.
\newblock \emph{Journal of machine learning research}, 12(7).

\bibitem[{Fallah et~al.(2020)Fallah, Mokhtari, and Ozdaglar}]{fallah2020personalized}
Alireza Fallah, Aryan Mokhtari, and Asuman Ozdaglar. 2020.
\newblock Personalized {F}ederated {L}earning: A meta-learning approach.
\newblock \emph{arXiv preprint arXiv:2002.07948}.

\bibitem[{Fan et~al.(2021)Fan, Bhosale, Schwenk, Ma, El{-}Kishky, Goyal, Baines, Celebi, Wenzek, Chaudhary, Goyal, Birch, Liptchinsky, Edunov, Auli, and Joulin}]{m2m100}
Angela Fan, Shruti Bhosale, Holger Schwenk, Zhiyi Ma, Ahmed El{-}Kishky, Siddharth Goyal, Mandeep Baines, Onur Celebi, Guillaume Wenzek, Vishrav Chaudhary, Naman Goyal, Tom Birch, Vitaliy Liptchinsky, Sergey Edunov, Michael Auli, and Armand Joulin. 2021.
\newblock Beyond english-centric multilingual machine translation.
\newblock \emph{J. Mach. Learn. Res.}, 22:107:1--107:48.

\bibitem[{Fan et~al.(2020)Fan, Bhosale, Schwenk, Ma, El-Kishky, Goyal, Baines, Celebi, Wenzek, Chaudhary, Goyal, Birch, Liptchinsky, Edunov, Grave, Auli, and Joulin}]{fan2020englishcentric}
Angela Fan, Shruti Bhosale, Holger Schwenk, Zhiyi Ma, Ahmed El-Kishky, Siddharth Goyal, Mandeep Baines, Onur Celebi, Guillaume Wenzek, Vishrav Chaudhary, Naman Goyal, Tom Birch, Vitaliy Liptchinsky, Sergey Edunov, Edouard Grave, Michael Auli, and Armand Joulin. 2020.
\newblock \href {https://www.jmlr.org/papers/volume22/20-1307/20-1307.pdf} {Beyond {E}nglish-centric multilingual machine translation}.
\newblock \emph{Journal of Machine Learning Research}, 22:1--48.

\bibitem[{Goyal et~al.(2022)Goyal, Gao, Chaudhary, Chen, Wenzek, Ju, Krishnan, Ranzato, Guzm{\'a}n, and Fan}]{goyal-etal-2022-flores}
Naman Goyal, Cynthia Gao, Vishrav Chaudhary, Peng-Jen Chen, Guillaume Wenzek, Da~Ju, Sanjana Krishnan, Marc{'}Aurelio Ranzato, Francisco Guzm{\'a}n, and Angela Fan. 2022.
\newblock \href {https://doi.org/10.1162/tacl_a_00474} {The {F}lores-101 evaluation benchmark for low-resource and multilingual machine translation}.
\newblock \emph{Transactions of the Association for Computational Linguistics}, 10:522--538.

\bibitem[{Goyal et~al.(2020)Goyal, Kumar, and Sharma}]{goyal-etal-2020-efficient}
Vikrant Goyal, Sourav Kumar, and Dipti~Misra Sharma. 2020.
\newblock \href {https://doi.org/10.18653/v1/2020.acl-srw.22} {Efficient neural machine translation for low-resource languages via exploiting related languages}.
\newblock In \emph{Proceedings of the 58th Annual Meeting of the Association for Computational Linguistics: Student Research Workshop}, pages 162--168, Online. Association for Computational Linguistics.

\bibitem[{Haddow et~al.(2022)Haddow, Bawden, Barone, Helcl, and Birch}]{10.1162/coli_a_00446}
Barry Haddow, Rachel Bawden, Antonio Valerio~Miceli Barone, Jindřich Helcl, and Alexandra Birch. 2022.
\newblock \href {https://doi.org/10.1162/coli_a_00446} {{Survey of Low-Resource Machine Translation}}.
\newblock \emph{Computational Linguistics}, 48(3):673--732.

\bibitem[{Hanzely et~al.(2020)Hanzely, Hanzely, Horv{\'a}th, and Richt{\'a}rik}]{hanzely2020lower}
Filip Hanzely, Slavom{\'\i}r Hanzely, Samuel Horv{\'a}th, and Peter Richt{\'a}rik. 2020.
\newblock Lower bounds and optimal algorithms for personalized federated learning.
\newblock \emph{Advances in Neural Information Processing Systems}, 33:2304--2315.

\bibitem[{Hedderich et~al.(2021)Hedderich, Lange, Adel, Str{\"o}tgen, and Klakow}]{hedderich-etal-2021-survey}
Michael~A. Hedderich, Lukas Lange, Heike Adel, Jannik Str{\"o}tgen, and Dietrich Klakow. 2021.
\newblock \href {https://doi.org/10.18653/v1/2021.naacl-main.201} {A survey on recent approaches for natural language processing in low-resource scenarios}.
\newblock In \emph{Proceedings of the 2021 Conference of the North American Chapter of the Association for Computational Linguistics: Human Language Technologies}, pages 2545--2568, Online. Association for Computational Linguistics.

\bibitem[{Hinton et~al.(2012)Hinton, Srivastava, and Swersky}]{hinton2012neural}
Geoffrey Hinton, Nitish Srivastava, and Kevin Swersky. 2012.
\newblock Neural networks for machine learning lecture 6a overview of mini-batch gradient descent.
\newblock \emph{Cited on}, 14(8):2.

\bibitem[{Huo et~al.(2024)Huo, Feng, Huang, Fu, Wang, and Qin}]{huo-etal-2024-gradient-consistency}
Wenshuai Huo, Xiaocheng Feng, Yichong Huang, Chengpeng Fu, Hui Wang, and Bing Qin. 2024.
\newblock \href {https://aclanthology.org/2024.lrec-main.696} {Gradient consistency-based parameter allocation for multilingual neural machine translation}.
\newblock In \emph{Proceedings of the 2024 Joint International Conference on Computational Linguistics, Language Resources and Evaluation (LREC-COLING 2024)}, pages 7901--7912, Torino, Italia. ELRA and ICCL.

\bibitem[{ImaniGooghari et~al.(2023)ImaniGooghari, Lin, Kargaran, Severini, Jalili~Sabet, Kassner, Ma, Schmid, Martins, Yvon, and Sch{\"u}tze}]{imanigooghari-etal-2023-glot500}
Ayyoob ImaniGooghari, Peiqin Lin, Amir~Hossein Kargaran, Silvia Severini, Masoud Jalili~Sabet, Nora Kassner, Chunlan Ma, Helmut Schmid, Andr{\'e} Martins, Fran{\c{c}}ois Yvon, and Hinrich Sch{\"u}tze. 2023.
\newblock \href {https://doi.org/10.18653/v1/2023.acl-long.61} {Glot500: Scaling multilingual corpora and language models to 500 languages}.
\newblock In \emph{Proceedings of the 61st Annual Meeting of the Association for Computational Linguistics (Volume 1: Long Papers)}, pages 1082--1117, Toronto, Canada. Association for Computational Linguistics.

\bibitem[{Kairouz et~al.(2021)Kairouz, McMahan, Avent, Bellet, Bennis, Bhagoji, Bonawitz, Charles, Cormode, Cummings et~al.}]{kairouz2021advances}
Peter Kairouz, H~Brendan McMahan, Brendan Avent, Aur{\'e}lien Bellet, Mehdi Bennis, Arjun~Nitin Bhagoji, Kallista Bonawitz, Zachary Charles, Graham Cormode, Rachel Cummings, et~al. 2021.
\newblock Advances and open problems in federated learning.
\newblock \emph{Foundations and trends{\textregistered} in machine learning}, 14(1--2):1--210.

\bibitem[{Keskar et~al.(2017)Keskar, Mudigere, Nocedal, Smelyanskiy, and Tang}]{keskar2017largebatch}
Nitish~Shirish Keskar, Dheevatsa Mudigere, Jorge Nocedal, Mikhail Smelyanskiy, and Ping Tak~Peter Tang. 2017.
\newblock On large-batch training for deep learning: Generalization gap and sharp minima.
\newblock In \emph{5th International Conference on Learning Representations, {ICLR} 2017}, Toulon, France.

\bibitem[{Kingma and Ba(2015)}]{kingma2014adam}
Diederik~P. Kingma and Jimmy Ba. 2015.
\newblock \href {http://arxiv.org/abs/1412.6980} {Adam: {A} method for stochastic optimization}.
\newblock In \emph{3rd International Conference on Learning Representations, {ICLR}}, San Diego, CA, USA.

\bibitem[{Konecn{\`y} et~al.(2016)Konecn{\`y}, McMahan, Yu, Richt{\'a}rik, Suresh, and Bacon}]{konecny2016federated}
Jakub Konecn{\`y}, H~Brendan McMahan, Felix~X Yu, Peter Richt{\'a}rik, Ananda~Theertha Suresh, and Dave Bacon. 2016.
\newblock Federated {L}earning: Strategies for improving communication efficiency.
\newblock \emph{arXiv preprint arXiv:1610.05492}, 8.

\bibitem[{Krasadakis et~al.(2024)Krasadakis, Sakkopoulos, and Verykios}]{electronics13030648}
Panteleimon Krasadakis, Evangelos Sakkopoulos, and Vassilios~S. Verykios. 2024.
\newblock \href {https://doi.org/10.3390/electronics13030648} {A survey on challenges and advances in natural language processing with a focus on legal informatics and low-resource languages}.
\newblock \emph{Electronics}, 13(3).

\bibitem[{Kulkarni et~al.(2020)Kulkarni, Kulkarni, and Pant}]{kulkarni2020survey}
Viraj Kulkarni, Milind Kulkarni, and Aniruddha Pant. 2020.
\newblock Survey of personalization techniques for federated learning.
\newblock In \emph{2020 fourth world conference on smart trends in systems, security and sustainability (WorldS4)}, pages 794--797. IEEE.

\bibitem[{Liao et~al.(2021)Liao, Khadivi, and Hewavitharana}]{liao-etal-2021-back}
Baohao Liao, Shahram Khadivi, and Sanjika Hewavitharana. 2021.
\newblock \href {https://aclanthology.org/2021.wmt-1.50} {Back-translation for large-scale multilingual machine translation}.
\newblock In \emph{Proceedings of the Sixth Conference on Machine Translation}, pages 418--424, Online. Association for Computational Linguistics.

\bibitem[{Lin et~al.(2024)Lin, Ji, Tiedemann, Martins, and Sch{\"{u}}tze}]{DBLP:journals/corr/abs-2401-13303}
Peiqin Lin, Shaoxiong Ji, J{\"{o}}rg Tiedemann, Andr{\'{e}} F.~T. Martins, and Hinrich Sch{\"{u}}tze. 2024.
\newblock \href {https://arxiv.org/html/2401.13303v1} {{MaLA}-500: Massive language adaptation of large language models}.
\newblock \emph{CoRR}, abs/2401.13303.

\bibitem[{Logacheva et~al.(2020)Logacheva, Teslenko, Shelmanov, Remus, Ustalov, Kutuzov, Artemova, Biemann, Ponzetto, and Panchenko}]{DBLP:conf/lrec/LogachevaTSRUKA20}
Varvara Logacheva, Denis Teslenko, Artem Shelmanov, Steffen Remus, Dmitry Ustalov, Andrey Kutuzov, Ekaterina Artemova, Chris Biemann, Simone~Paolo Ponzetto, and Alexander Panchenko. 2020.
\newblock \href {https://aclanthology.org/2020.lrec-1.728} {Word sense disambiguation for 158 languages using word embeddings only}.
\newblock In \emph{Proceedings of the Twelfth Language Resources and Evaluation Conference}, pages 5943--5952, Marseille, France. European Language Resources Association.

\bibitem[{Ma et~al.(2021)Ma, Dong, Huang, Zhang, Muzio, Singhal, Awadalla, Song, and Wei}]{DBLP:journals/corr/abs-2106-13736}
Shuming Ma, Li~Dong, Shaohan Huang, Dongdong Zhang, Alexandre Muzio, Saksham Singhal, Hany~Hassan Awadalla, Xia Song, and Furu Wei. 2021.
\newblock Deltalm: Encoder-decoder pre-training for language generation and translation by augmenting pretrained multilingual encoders.
\newblock \emph{CoRR}, abs/2106.13736.

\bibitem[{McMahan et~al.(2017)McMahan, Moore, Ramage, Hampson, and y~Arcas}]{mcmahan2017communication}
Brendan McMahan, Eider Moore, Daniel Ramage, Seth Hampson, and Blaise~Aguera y~Arcas. 2017.
\newblock Communication-efficient learning of deep networks from decentralized data.
\newblock In \emph{Artificial intelligence and statistics}, pages 1273--1282. PMLR.

\bibitem[{Millour et~al.(2024)Millour, Brasile, Ghia, and Kevers}]{millour-etal-2024-agettivu-aggitivu}
Alice Millour, Lorenza Brasile, Alberto Ghia, and Laurent Kevers. 2024.
\newblock \href {https://aclanthology.org/2024.lrec-main.52} {Agettivu, aggitivu o aghjettivu? {POS} tagging {C}orsican dialects}.
\newblock In \emph{Proceedings of the 2024 Joint International Conference on Computational Linguistics, Language Resources and Evaluation (LREC-COLING 2024)}, pages 600--608, Torino, Italia. ELRA and ICCL.

\bibitem[{Nemirovski and Yudin(1983)}]{nemirovsky1983problem}
Arkadi~Semenovich Nemirovski and David~Borisovich Yudin. 1983.
\newblock \emph{Problem Complexity and Method Efficiency in Optimization}.
\newblock A Wiley-Interscience publication. Wiley.

\bibitem[{Nemirovskij and Yudin(1983)}]{nemirovskij1983problem}
Arkadij~Semenovi{\v{c}} Nemirovskij and David~Borisovich Yudin. 1983.
\newblock Problem complexity and method efficiency in optimization.

\bibitem[{Paszke et~al.(2019)Paszke, Gross, Massa, Lerer, Bradbury, Chanan, Killeen, Lin, Gimelshein, Antiga, Desmaison, Köpf, Yang, DeVito, Raison, Tejani, Chilamkurthy, Steiner, Fang, Bai, and Chintala}]{paszke2019pytorch}
Adam Paszke, Sam Gross, Francisco Massa, Adam Lerer, James Bradbury, Gregory Chanan, Trevor Killeen, Zeming Lin, Natalia Gimelshein, Luca Antiga, Alban Desmaison, Andreas Köpf, Edward Yang, Zach DeVito, Martin Raison, Alykhan Tejani, Sasank Chilamkurthy, Benoit Steiner, Lu~Fang, Junjie Bai, and Soumith Chintala. 2019.
\newblock \href {https://arxiv.org/abs/1912.01703} {Pytorch: An imperative style, high-performance deep learning library}.
\newblock \emph{Preprint}, arXiv:1912.01703.

\bibitem[{Post(2018)}]{post2018call}
Matt Post. 2018.
\newblock \href {https://doi.org/10.18653/v1/W18-6319} {A call for clarity in reporting {BLEU} scores}.
\newblock In \emph{Proceedings of the Third Conference on Machine Translation: Research Papers}, pages 186--191, Brussels, Belgium. Association for Computational Linguistics.

\bibitem[{Ranathunga et~al.(2023)Ranathunga, Lee, Skenduli, Shekhar, Alam, and Kaur}]{DBLP:journals/csur/RanathungaLSSAK23}
Surangika Ranathunga, En{-}Shiun~Annie Lee, Marjana~Prifti Skenduli, Ravi Shekhar, Mehreen Alam, and Rishemjit Kaur. 2023.
\newblock Neural machine translation for low-resource languages: {A} survey.
\newblock \emph{{ACM} Comput. Surv.}, 55(11):229:1--229:37.

\bibitem[{Robbins and Monro(1951)}]{robbins1951stochastic}
Herbert Robbins and Sutton Monro. 1951.
\newblock A stochastic approximation method.
\newblock \emph{The annals of mathematical statistics}, pages 400--407.

\bibitem[{Schr{\"o}der and Biemann(2020)}]{schroder-biemann-2020-estimating}
Fynn Schr{\"o}der and Chris Biemann. 2020.
\newblock \href {https://doi.org/10.18653/v1/2020.acl-main.268} {Estimating the influence of auxiliary tasks for multi-task learning of sequence tagging tasks}.
\newblock In \emph{Proceedings of the 58th Annual Meeting of the Association for Computational Linguistics}, pages 2971--2985, Online. Association for Computational Linguistics.

\bibitem[{Shalev-Shwartz and Ben-David(2014)}]{shalev2014understanding}
Shai Shalev-Shwartz and Shai Ben-David. 2014.
\newblock \emph{Understanding machine learning: From theory to algorithms}.
\newblock Cambridge university press.

\bibitem[{Shalev-Shwartz et~al.(2009)Shalev-Shwartz, Shamir, Srebro, and Sridharan}]{shalev2009stochastic}
Shai Shalev-Shwartz, Ohad Shamir, Nathan Srebro, and Karthik Sridharan. 2009.
\newblock Stochastic convex optimization.
\newblock In \emph{COLT}, volume~2, page~5.

\bibitem[{Streeter and McMahan(2010)}]{streeter2010less}
Matthew Streeter and H.~Brendan McMahan. 2010.
\newblock Less regret via online conditioning.
\newblock \emph{arXiv preprint arXiv:1002.4862}.

\bibitem[{Sutawika and Cruz(2021)}]{sutawika2021data}
Lintang Sutawika and Jan Christian~Blaise Cruz. 2021.
\newblock \href {https://aclanthology.org/2021.wmt-1.52} {Data processing matters: {SRPH}-konvergen {AI}{'}s machine translation system for {WMT}{'}21}.
\newblock In \emph{Proceedings of the Sixth Conference on Machine Translation}, pages 431--438, Online. Association for Computational Linguistics.

\bibitem[{Tars et~al.(2022)Tars, T{\"{a}}ttar, and Fishel}]{DBLP:journals/bjmc/TarsTF22}
Maali Tars, Andre T{\"{a}}ttar, and Mark Fishel. 2022.
\newblock Cross-lingual transfer from large multilingual translation models to unseen under-resourced languages.
\newblock \emph{Balt. J. Mod. Comput.}, 10(3).

\bibitem[{Tupitsa et~al.(2024)Tupitsa, Horv{\'a}th, Tak{\'a}{\v{c}}, and Gorbunov}]{tupitsa2024federated}
Nazarii Tupitsa, Samuel Horv{\'a}th, Martin Tak{\'a}{\v{c}}, and Eduard Gorbunov. 2024.
\newblock \href {https://arxiv.org/abs/2402.05050} {Federated learning can find friends that are beneficial}.
\newblock \emph{arXiv preprint arXiv:2402.05050}.

\bibitem[{Wang et~al.(2021)Wang, Tan, Luo, Qin, and Liu}]{DBLP:conf/ijcai/Wang0LQL21}
Rui Wang, Xu~Tan, Renqian Luo, Tao Qin, and Tie{-}Yan Liu. 2021.
\newblock A survey on low-resource neural machine translation.
\newblock In \emph{{IJCAI}}, pages 4636--4643. ijcai.org.

\bibitem[{Wang et~al.(2020)Wang, Tsvetkov, and Neubig}]{wang-etal-2020-balancing}
Xinyi Wang, Yulia Tsvetkov, and Graham Neubig. 2020.
\newblock \href {https://doi.org/10.18653/v1/2020.acl-main.754} {Balancing training for multilingual neural machine translation}.
\newblock In \emph{Proceedings of the 58th Annual Meeting of the Association for Computational Linguistics}, pages 8526--8537, Online. Association for Computational Linguistics.

\bibitem[{Ward et~al.(2019)Ward, Wu, and Bottou}]{ward2018adagrad}
Rachel Ward, Xiaoxia Wu, and Leon Bottou. 2019.
\newblock \href {https://proceedings.mlr.press/v97/ward19a.html} {{A}da{G}rad stepsizes: Sharp convergence over nonconvex landscapes}.
\newblock In \emph{Proceedings of the 36th International Conference on Machine Learning}, volume~97 of \emph{Proceedings of Machine Learning Research}, pages 6677--6686. PMLR.

\bibitem[{Wenzek et~al.(2021)Wenzek, Chaudhary, Fan, Gomez, Goyal, Jain, Kiela, Thrush, and Guzm{\'a}n}]{wenzek-etal-2021-findings}
Guillaume Wenzek, Vishrav Chaudhary, Angela Fan, Sahir Gomez, Naman Goyal, Somya Jain, Douwe Kiela, Tristan Thrush, and Francisco Guzm{\'a}n. 2021.
\newblock \href {https://aclanthology.org/2021.wmt-1.2} {Findings of the {WMT} 2021 shared task on large-scale multilingual machine translation}.
\newblock In \emph{Proceedings of the Sixth Conference on Machine Translation}, pages 89--99, Online. Association for Computational Linguistics.

\bibitem[{Wolf et~al.(2020)Wolf, Debut, Sanh, Chaumond, Delangue, Moi, Cistac, Rault, Louf, Funtowicz, Davison, Shleifer, von Platen, Ma, Jernite, Plu, Xu, Scao, Gugger, Drame, Lhoest, and Rush}]{wolf-etal-2020-transformers}
Thomas Wolf, Lysandre Debut, Victor Sanh, Julien Chaumond, Clement Delangue, Anthony Moi, Pierric Cistac, Tim Rault, Rémi Louf, Morgan Funtowicz, Joe Davison, Sam Shleifer, Patrick von Platen, Clara Ma, Yacine Jernite, Julien Plu, Canwen Xu, Teven~Le Scao, Sylvain Gugger, Mariama Drame, Quentin Lhoest, and Alexander~M. Rush. 2020.
\newblock \href {https://www.aclweb.org/anthology/2020.emnlp-demos.6} {Transformers: State-of-the-art natural language processing}.
\newblock In \emph{Proceedings of the 2020 Conference on Empirical Methods in Natural Language Processing: System Demonstrations}, pages 38--45, Online. Association for Computational Linguistics.

\bibitem[{Wu and Wang(2021)}]{wu2021fast}
Hongda Wu and Ping Wang. 2021.
\newblock Fast-convergent federated learning with adaptive weighting.
\newblock \emph{IEEE Transactions on Cognitive Communications and Networking}, 7(4):1078--1088.

\bibitem[{Xie et~al.(2021)Xie, Hu, Yang, Yu, and Ju}]{xie-etal-2021-tentrans}
Wanying Xie, Bojie Hu, Han Yang, Dong Yu, and Qi~Ju. 2021.
\newblock \href {https://aclanthology.org/2021.wmt-1.53} {{T}en{T}rans large-scale multilingual machine translation system for {WMT}21}.
\newblock In \emph{Proceedings of the Sixth Conference on Machine Translation}, pages 439--445, Online. Association for Computational Linguistics.

\bibitem[{Yang et~al.(2021)Yang, Ma, Huang, Zhang, Dong, Huang, Muzio, Singhal, Hassan, Song, and Wei}]{yang-etal-2021-multilingual-machine}
Jian Yang, Shuming Ma, Haoyang Huang, Dongdong Zhang, Li~Dong, Shaohan Huang, Alexandre Muzio, Saksham Singhal, Hany Hassan, Xia Song, and Furu Wei. 2021.
\newblock \href {https://aclanthology.org/2021.wmt-1.54} {Multilingual machine translation systems from {M}icrosoft for {WMT}21 shared task}.
\newblock In \emph{Proceedings of the Sixth Conference on Machine Translation}, pages 446--455, Online. Association for Computational Linguistics.

\bibitem[{Yankovskaya et~al.(2023)Yankovskaya, Tars, T{\"a}ttar, and Fishel}]{yankovskaya-etal-2023-machine}
Lisa Yankovskaya, Maali Tars, Andre T{\"a}ttar, and Mark Fishel. 2023.
\newblock \href {https://aclanthology.org/2023.nodalida-1.77} {Machine translation for low-resource {F}inno-{U}gric languages}.
\newblock In \emph{Proceedings of the 24th Nordic Conference on Computational Linguistics (NoDaLiDa)}, pages 762--771, T{\'o}rshavn, Faroe Islands. University of Tartu Library.

\bibitem[{Zaheer et~al.(2018)Zaheer, Reddi, Sachan, Kale, and Kumar}]{zaheer2018adaptive}
Manzil Zaheer, Sashank Reddi, Devendra Sachan, Satyen Kale, and Sanjiv Kumar. 2018.
\newblock Adaptive methods for nonconvex optimization.
\newblock \emph{Advances in neural information processing systems}, 31.

\bibitem[{Zhang et~al.(2020)Zhang, Karimireddy, Veit, Kim, Reddi, Kumar, and Sra}]{zhang2020adaptive}
Jingzhao Zhang, Sai~Praneeth Karimireddy, Andreas Veit, Seungyeon Kim, Sashank Reddi, Sanjiv Kumar, and Suvrit Sra. 2020.
\newblock Why are adaptive methods good for attention models?
\newblock \emph{Advances in Neural Information Processing Systems}, 33:15383--15393.

\end{thebibliography}

\appendix
\onecolumn

\section{Dataset and Model  Details}\label{sec:appendix_params}

In addition to dataset scaling, we also add a preprocessing step: from a deeper look into the data, we can see that some translations contain code snippets, HTML, and pairs containing different addresses and numbers in the input language and output language. To avoid such data, we filter the sentences in the training set so that (i) input and output length in tokens is not less than $5$ tokens and not larger than $256$ tokens, as only a negligible portion of the data exceeded this limit; (ii) we keep sentences with the numbers matching in both input and output; (iii) we keep alphanumeric sentences with basic punctuation only. We also check that both datasets we apply do not contain personally identifying information or offensive content.

We used M2M100 as a base model \cite{m2m100}, MIT Licensed. In CP setting, we pretrained all models on all languages for a maximum of 10 epochs, with the best-performing checkpoint selected for later fine-tuning. Fine-tuning was conducted for up to 60 epochs, and the best-performing checkpoint was reported. The \algname{MeritOpt} model was trained until the score stopped improving, with a maximum computation time of four days. The maximum number of epochs was limited by the amount of available computational resources, as well to perform comparable or even more steps than previous studies on similar datasets \cite{DBLP:journals/bjmc/TarsTF22,sutawika2021data}.

Training parameters included a fixed batch size of 64 and a learning rate of 3e-5. We used a Cosine Annealing Scheduler with a minimum learning rate of 1e-5. The baseline optimizer was \algname{Adam}, with $\beta_1 = 0.9$ and $\beta_2 = 0.98$, in line with previous studies \cite{xie-etal-2021-tentrans}. 

Our implementation primarily relied on PyTorch \cite{paszke2019pytorch} and Transformers \cite{wolf-etal-2020-transformers} libraries. All our artifacts are licensed under Apache 2.0.

\begin{table}[h!]
\centering
\begin{tabular}{lllllll}
\toprule
\multirow{2}{*}{Input Language} & \multirow{2}{*}{Total} & \multicolumn{3}{c}{Filtered Train} & \multirow{2}{*}{Val} & \multirow{2}{*}{Test} \\
\cline{3-5}
                          &                        & Small   & Medium   & Large   &                      &                       \\
\midrule
Indonesian                & 54M                    & 37K     & 74K      & 259K    & 1K                   & 1K                    \\
Malay                     & 13M                    & 26K     & 53K      & 185K    & 1K                   & 1K                    \\
Tagalog                   & 2M                     & 10K     & 20K      & 70K     & 1K                   & 1K                    \\
Tamil                     & 13M                    & 5K      & 10K      & 35K     & 1K                   & 1K                    \\
Javanese                  & 3M                     & 776     & 1.5K     & 5K      & 1K                   & 1K              \\
\bottomrule
\end{tabular}
\caption{Dataset statistics for South East Asian languages. Total denotes the original dataset size in sequences, Filtered small, medium and large train are the subsets used for experiments.}
\label{tab:stats_ind}
\end{table}

\begin{table}[h!]
\centering
\begin{tabular}{llll}
\toprule
\multicolumn{1}{c}{\multirow{2}{*}{Input Language}} & \multicolumn{1}{c}{\multirow{2}{*}{Train}} & \multicolumn{1}{c}{\multirow{2}{*}{Val}} & \multicolumn{1}{c}{\multirow{2}{*}{Test}} \\
\multicolumn{1}{c}{}                          & \multicolumn{1}{c}{}                       & \multicolumn{1}{c}{}                     & \multicolumn{1}{c}{}                      \\
\midrule
North Sami                                           & 61,559                                     & 200                                      & 500                                       \\
Inari Sami                                           & 8,750                                      & 200                                      & 500                                       \\
Skolt Sami                                           & 1,998                                      & 200                                      & 500                                       \\
South Sami                                           & 1,734                                      & 200                                      & 500                            \\
\bottomrule
\end{tabular}
\caption{Dataset statistics for Finno-Samic languages.}
\label{tab:stats_sami}
\end{table}

\clearpage

\section{Illustrative Experiment with Mean Estimation Problem}\label{appendix:mean_estimation}

In this section, we provide an illustrative experiment with the mean estimation problem. That is, the goal is to solve the following minimization problem:
\begin{equation*}
    \min\limits_{x\in \R^d} \left\{f_{\cD}(x) := \E_{\xi \sim \cD}[\|x - \xi\|^2] \right\},
\end{equation*}
where $\cD := \cN(0, \mI)$ is a standard Gaussian distribution. One can show that the optimal value equals $\E_{\xi\sim \cD}\|\xi\|^2 = d$, which is attained at $x^* = 0$. Next, we consider three datasets: $\mD_1$ is sampled from the target distribution $\cD$, $\mD_2$ is sampled from close distribution $\cN(\mu \mathbf{1}, \mI)$, where $\mu = 0.0001$ and $\mathbf{1} := (1, \ldots, 1)^\top \R^d$, and $\mD_3$ is sampled from quite different distribution $\cN(e, \mI)$, where $e$ is some randomly precomputed unit vector. The sizes of the input dataset are: $|\mD_1| = 20$, $|\mD_2| = 1000$, $|\mD_3| = 1000$. Therefore, this situation resembles training for the low-resource language, when two high-resource languages are available. We take mini-batch of $10\%$ for each dataset to compute $g_i(x^t)$ in \algname{MeritOpt} and use simple \algname{SGD} as $\texttt{OptStep}$. Target validation dataset $\widehat\mD$ is sampled from $\cD$ (same distribution as for $\mD_1$) and has size $|\widehat\mD| = 100$ (though only mini-batch of $10$ samples from $\widehat\mD$ is used at each iteration to perform a computation of aggregation weight $w^{t+1}$). To solve the problem in Line~\ref{lst:line:aux_problem}, we run \algname{MD} with learning rate $10$.

\begin{figure}[h!]
   % \centering
    \begin{minipage}[h!]{0.45\textwidth}
        \centering
        \includegraphics[width=1\linewidth]{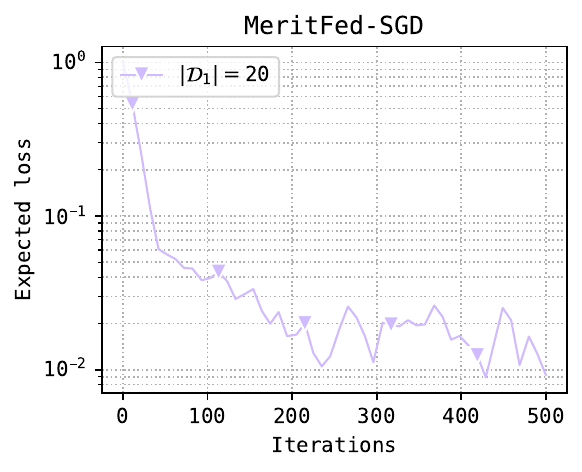}
        % \vspace{-0.6cm}
        \vspace{-0.8cm}
        % \caption{Mean Estimation: $\mu = 0.0001$, MD learning rate = 10.}
        \label{fig:emean1}
    \end{minipage}
% \end{figure}
\hfill
% \begin{figure}[t]
%     \centering
    \begin{minipage}[h!]{0.45\textwidth}
        \centering
        \includegraphics[width=1\linewidth]{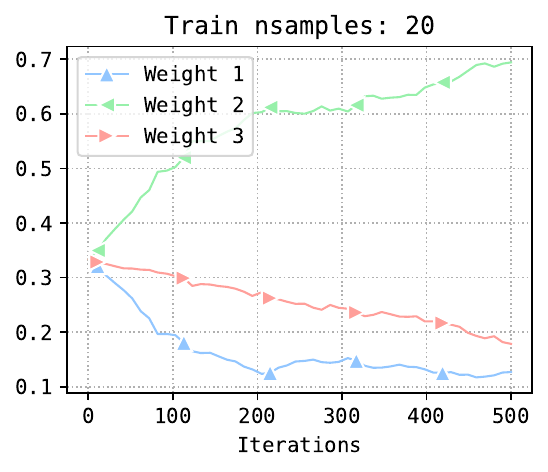}
        % \vspace{-0.6cm}
        \vspace{-0.8cm}
        % \caption{Mean Estimation: $\mu = 0.01$, MD learning rate = 4.5.}
        \label{fig:emean2}
    \end{minipage}
% \end{figure}
% \hfill
% % \begin{figure}[t]
% %     \centering
%     \begin{minipage}[h!]{0.325\textwidth}
%         \centering
%         \includegraphics[width=1\linewidth]{img/meanest/30 model=Mean lr=0.008.pdf}
%         % \vspace{-0.6cm}
%         \vspace{-0.8cm}
        
%         \label{fig:emean3}
%     \end{minipage}
\caption{Mean Estimation: $\mu = 0.0001$, MD learning rate = 10.}
\label{fig:mean_estimation}
\end{figure}

The results are presented in Figure~\ref{fig:mean_estimation}. We see that the weight for the first and the third datasets decrease during the training, while the weight for the second dataset increases and remains the largest one. Such a behavior is natural since the batchsize for the target dataset is much smaller than for the second dataset ($2$ and $100$ respectively) and since the second dataset comes from very close distribution to the target one it is more beneficial to use slightly biased but less noisy updates from the second dataset than unbiased but noisy updates from the first dataset. As for the third dataset, its weight decreases since it comes from completely different distribution.

Overall, the result of this experiment are quite consistent with the ones we obtained for Javanese language where the weight for the target language also becomes the smallest after certain number of steps and the highest weight is assigned to close but different language (Indonesian), see Figure~\ref{fig:weight_ind}.

\clearpage

\section{Technical Details and Theoretical Results: Complete Statements and Proofs}\label{appendix:theory}

\subsection{Further Details on Mirror Descent}\label{appendix:MD}

As we explain in Section~\ref{section:methodology}, the aggregation weights are obtained at each step via approximately minimizing validation loss as a function of the aggregation weights. More precisely, in Line~\ref{lst:line:aux_problem} of Algorithm~\ref{alg:meritfed}, the goal is to minimize function $\varphi(w)$ defined as
\begin{equation}
    \varphi(w) = f_{\hat D}\left( \texttt{OptStep}\left( x^t, \sum\limits_{i=1}^n w_i g_i(x^t),  \gamma_t \right) \right), \label{eq:MD_objective}
\end{equation}
where $f_{\hat D}$ is a validation loss for the target language, $\texttt{OptStep}$ is a step of optimization method (e.g., \algname{Adam}), $x^t$ are model parameters after $t$ steps, $\gamma_t$ is a learning rate at step $t$, and $g_i(x^t)$ is a stochastic gradient corresponding to the language $i$. Since Stochastic Mirror Descent (SMD) \citep{nemirovsky1983problem} with Kulback-Leibler distance as a Bregman divergence is a natural choice for the minimization on the probability simplex, which is our case, we use SMD to minimize $\varphi(w)$ on the simplex. The update rule of SMD, in this settings, can be written \citep[Chapter 9]{beck2017first} as
\begin{equation}
    w^{k+1} = \frac{w^k \exp(-\eta \tilde\nabla \varphi(w^k))}{\sum\limits_{i=1}^n w_i^k \exp(-\eta [\tilde\nabla \varphi(w^k))]_i}, \label{eq:update_rule_SMD}
\end{equation}
where $\eta$ is a learning rate for SMD, $\tilde\nabla \varphi(w^k))$ is a stochastic gradient of $\varphi(w)$ for the current weights $w^k$, product of vectors $w^k \exp(-\eta \tilde\nabla \varphi(w^k))$ is computed coordinate-wise, and $[\tilde\nabla \varphi(w^k))]_i$ is the $i$-th component of $\tilde\nabla \varphi(w^k))$.

\subsection{Preliminaries}

In this section, we provide the details on the theoretical convergence results for \algname{MeritOpt}. For notational convenience, we assume that $\mD_i$ comes from distribution $\cD_i$ and denote the corresponding expected loss function as $f_i$ for all $i=1,\ldots, n$. Therefore, according to the introduced notation $f_{1}$ and $f_{\cD}$ denote the same loss function. Similarly to the setup considered by \citet{tupitsa2024federated}, we denote the set of indices such that $\cD_i = \cD_1$: $\cG := \{i\in \{1,\ldots,n\}\mid \cD_i = \cD_1\}$. In other words, for every $i \in \cG$ dataset $\mD_i$ comes from the target distribution and, thus, should be beneficial for the training.

Next, we make the following standard assumption about the stochastic gradients.

\begin{assumption}\label{as:bounded-var}
	% For all $\x\in \R^d$ the noise in every stochastic estimator is independent. 
 For all $i\in \cG$ the stochastic gradient $\g_i(\x)$ is an unbiased estimator of $\nabla\f_i(\x)$ with bounded variance, i.e., $\E_{\xiv_i\sim \cD_i} [\g_i(\x)] = \nabla\f_i(\x)$ and for some $\sigma \geq 0$
    %  \begin{equation} \label{ass:xi-unbiased}
    %     \E_{\xiv_i} [\f_i(\x)] = \F_i(\x),
    % \end{equation}
    % and has bounded variance, i.e., satisfies 
    \begin{equation}  \label{eq:xi-var}
    \textstyle
        \E_{\xiv_i\sim \cD_i} \left[\norm*{ \g_i(\x) -  \nabla\f_i(\x)}^2\right] \leq \sigma^2.
    \end{equation}
\end{assumption}

Let $\w^\text{ideal}$ denote a weight vector containing equal non-zero weights only for the datasets from the target distribution. If Assumption~\ref{as:bounded-var} holds, then due to the independence of $\{g_i(x)\}_{i\in\cG}$
\begin{equation}  \label{eq:someeq}
\textstyle
    \E_{\xiv_i} \left[\norm*{ \sum\limits_{i=1}^n \w^\text{ideal}_i g_i(x) -  \nabla\f_1(\x)}^2\right] = \E_{\xiv_i} \left[\norm*{ \frac{1}{|\cG|}\sum\limits_{i\in \cG}  g_i(x) -  \nabla\f_1(\x)}^2\right] \leq \frac{\sigma^2}{\abs{\cG}} \equiv \sigma_*^2.
\end{equation}

We also assume that the objective is $L$-smooth
\begin{assumption}\label{as:lipschitzness}
    $\f_1$ is $L$-smooth, i.e., $\forall\; \x,\y \in \R^d$
    \begin{equation}\label{eq:lipschitzness}\textstyle
		\f_1\rbr{\x} \le \f_1\rbr{\y} + \inp*{\nabla \f_1\rbr{\y}}{\x - \y} + \frac{L}{2}\norm*{\x - \y}^2.
	\end{equation}
\end{assumption}

For the sake of brevity, we will also use the following notation:
\begin{eqnarray*}
    \x^{t+1}(\w) = \texttt{OptStep}\left(\x^t, \sum\limits_{i=1}^n \w_i \g_i(\x^t), \gamma_t\right).
\end{eqnarray*}

\subsection{Generic Scheme of the Proof}

The proof for \algname{MeritOpt-SGD} from \cite{tupitsa2024federated} is based on the assumption that the auxiliary problem can be solved with $\delta$ error:
\begin{eqnarray}\label{eq:aux_approx}
    \E\sbr*{\f_1(\x^{t+1}) | \x^t, \xiv^t} - \min\limits_{w\in \Delta_1^n}f_1\left(x^{t+1}(\w)\right) \le \delta,
\end{eqnarray}
and the following inequality
\begin{eqnarray}\label{eq:aux_opt}
    \min\limits_{w\in \Delta_1^n}f_1\left(x^{t+1}(\w)\right)  \le \f_1(x^{t+1}(\w^\text{ideal})),
\end{eqnarray}
which holds by the definition of the minimum. These two inequalities together imply
 \begin{eqnarray}\label{eq:aux_res}
    \E\sbr*{\f_1(x^{t+1}) | \x^t} \leq \E\sbr*{\f_1(x^{t+1}(\w^\text{ideal})) | \x^t} + \delta.
\end{eqnarray}
The rest of the proof for \algname{MeritOpt-SGD} follows the same scheme as for \algname{SGD} that uses $\sum_{i=1}^n \w^\text{ideal}_i g_i(x)$ as the stochastic gradient, i.e., as for the method $x^{t+1} = x^t - \gamma \sum_{i=1}^n \w^\text{ideal}_i g_i(x^t) = x^t - \frac{\gamma}{|\cG|}\sum_{i\in \cG} g_i(x^t)$.

We noticed, that convergence result of \algname{MeritOpt} envelope can be obtained in the case when the analysis of the method being enveloped uses only two subsequent points and relies on the analysis of the inequality $\E[f_1(x^{t+1})] \leq \E[f_1(x^t)] + \Delta_t$, where $\Delta_t$ is some additional iteration-dependent term. Then, using \eqref{eq:aux_res}, one can show that \algname{MeritOpt} version of the method decreases expected function value not less then the ideal update at each iteration (up to the error term of solving the problem in Line~\ref{lst:line:aux_problem}). In the next two subsections, we provide the results for \algname{MeritOpt-RMSProp} and \algname{MeritOpt-AdaGrad-Norm}.

\subsection{Special Case: \algname{RMSProp}}

In this subsection, we consider \algname{RMSProp} as $\texttt{OptStep}$, i.e.,
\begin{gather*}
    \texttt{OptStep}(x^t, g^t, \gamma_t) = x^t - \frac{\gamma_t}{b_t}g^t,\quad b_t = \sqrt{\beta_2 b_{t-1}^2 + (1-\beta_2)(g^t)^2} + \epsilon,
\end{gather*}
where all arithmetical operations (multiplication, division, summation, taking the square/square root) are coordinate-wise. We emphasize that \algname{RMSProp} can be seen as \algname{Adam} without momentum ($\beta_1 = 0$).

We base our proof on the one from \citep{zaheer2018adaptive}, that additionally uses the following assumption.
\begin{assumption}\label{as:boundedness}
Each component of the stochastic gradient $\g_i\rbr{\x}$ for $i \in \cG$ is bounded, \textit{i.e.}, 
\begin{equation}
    \norm*{\sbr*{\g_i(\x)}_j} \le G.
\end{equation}    
\end{assumption}

\begin{theorem}
    Let Assumptions~\ref{as:bounded-var}, \ref{as:lipschitzness}, \ref{as:boundedness} hold. If Line~\ref{lst:line:aux_problem} is solved with error $\delta \geq 0$ (see \eqref{eq:aux_approx}), then \algname{MeritOpt-RMSProp} with $\gamma_t = \gamma \leq \frac{\epsilon}{2L}$ and $\beta_2 \geq 1 - \frac{\epsilon^2}{16G^2}$ after $T$ iterations satisfy
    \begin{equation*}
        \min\limits_{t=0,\ldots, T-1}\E\|\nabla f_1(x^t)\|^2 \le 2\rbr{\sqrt{\beta_2} G + \epsilon} \times \sbr*{\frac{\rbr*{\f_1\rbr*{\x^{0}} - \f_1\rbr*{\x^{*}}}}{\gamma T} + \sigma_*^2 \rbr*{\frac{\gamma G \sqrt{1-\beta_2}}{\epsilon^2} + \frac{L \gamma^2}{2\epsilon^2}} + \frac{\delta}{\gamma}}.
    \end{equation*}
\end{theorem}
\begin{proof}
We start with the following inequality from the page 13  of~\cite{zaheer2018adaptive}
\begin{eqnarray*}
    \E\sbr{\f_1(x^{t+1}(\w^\text{ideal})) | \x^t} \leq \f_1(\x^{t}) - \frac{\gamma_t}{2\rbr*{\sqrt{\beta_2} G+ \epsilon}}  \norm*{\nabla\f_1(\x^{t})}^2  + \rbr*{\frac{\gamma_t G \sqrt{1-\beta_2}}{\epsilon^2} + \frac{L \gamma_t^2}{2\epsilon^2}} \sigma_*^2
\end{eqnarray*}
in a slightly adjusted form. In fact, this inequality holds for any $\x^t$ and ideally aggregated gradients $\sum\limits_{i=1}^n \w^\text{ideal}_i g_i(\x^t)$. Applying~\eqref{eq:aux_res}, we get 
\begin{eqnarray*}
    \E\sbr*{\f_1(\x^{t+1})| \x^t} \leq \f_1(\x^{t}) - \frac{\gamma_t}{2\rbr*{\sqrt{\beta_2} G+ \epsilon}}  \norm*{\nabla\f_1(\x^{t})}^2  + \rbr*{\frac{\gamma_t G \sqrt{1-\beta_2}}{\epsilon^2} + \frac{L \gamma_t^2}{2\epsilon^2}} \sigma_*^2 + \delta.
\end{eqnarray*}
Following the same steps of the rest of the proof from \citep{zaheer2018adaptive}, we obtain
\begin{eqnarray*}
    \frac{1}{T}\sum_{t=0}^{T-1}\E \norm*{\nabla\f_1\rbr*{\x^{t}}}^{2} \le 2\rbr{\sqrt{\beta_2} G + \epsilon} \times \sbr*{\frac{\rbr*{\f_1\rbr*{\x^{0}} - \f_1\rbr*{\x^{*}}}}{\gamma T} + \sigma_*^2 \rbr*{\frac{\gamma G \sqrt{1-\beta_2}}{\epsilon^2} + \frac{L \gamma^2}{2\epsilon^2}} + \frac{\delta}{\gamma}}, 
\end{eqnarray*}
where $\gamma_t = \gamma \leq \frac{\epsilon}{2L}$ is used. It remains to notice that $\min\limits_{t=0,\ldots, T-1}\E\|\nabla f_1(x^t)\|^2 \leq \frac{1}{T}\sum\limits_{t=0}^{T-1}\E \norm*{\nabla\f_1\rbr*{\x^{t}}}^{2}$.
\end{proof}

\subsection{Special Case: \algname{AdaGrad-Norm}}

In this subsection, we consider \algname{AdaGrad-Norm} \cite{ward2018adagrad} as $\texttt{OptStep}$, i.e.,
\begin{gather*}
    \texttt{OptStep}(x^t, g^t, \gamma_t) = x^t - \frac{\gamma_t}{b_{t+1}}g^t,\quad b_{t+1} = \sqrt{b_{t}^2 + \|g^t\|^2}.
\end{gather*}

We base our proof on the one from \citep{ward2018adagrad}, that additionally uses the following assumption.
\begin{assumption}\label{as:norm_bounded}
    Gradients $\nabla \f_i(\x)$ are uniformly bounded for $i \in \cG$:
    \begin{equation}
        \| \nabla \f_i(\x)\| \le G.
    \end{equation}
\end{assumption}

% For AdaGrad-Norm update rule writes as follows
% \begin{equation*}
%     x^{t+1}(\w) = x^{t} -  \frac{\gamma}{{b}_{t+1}(\w)}g^{t}(\w),
% \end{equation*}
% where $g^{t}(\w) = \sum\limits_{i=1}^n \w_i  \g_i(\x^{t}))$ and ${b}_{t+1}^2(\w) = {b}_{t}^2+ { \| g^{t}(\w) \|^2}$.

\begin{theorem}
    Let Assumptions~\ref{as:bounded-var}, \ref{as:lipschitzness}, \ref{as:norm_bounded} hold. If Line~\ref{lst:line:aux_problem} is solved with error $\delta \geq 0$ (see \eqref{eq:aux_approx}), then \algname{MeritOpt-AdaGrad-Norm} with after $T$ iterations satisfy
    \begin{align*}
      \min_{t \leq T} \left( \mathbb{E} \left[ \|\nabla \f_1(x^t)\|^{
   \frac{4}{3}} \right] \right)^\frac{3}{2}  \leq  \left( \frac{2b_0}{T} + \frac{4(G+\sigma_*)}{\sqrt{T}}\right)C_{F},
\end{align*}
where $$C_{F} = \frac{\delta T}{\gamma} + \frac{ \f_1(x^0) - \f_1(x^*) }{\gamma}+ \frac{4\sigma_*+\gamma L}{2}\log \left(\frac{20T\left( \sigma_*^2 + G^2\right)}{b_0^2} +10 \right). $$
\end{theorem}
\begin{proof}
    Notating $\widetilde g^{t} = \frac{1}{|\cG|}\sum_{i\in \cG}g_i(x^t)$  and $\widetilde b_{t+1} = \sqrt{b_{t}^2 + \|\widetilde g^t\|^2}$ we rewrite the first line of the main proof from \citep{ward2018adagrad} as
\begin{eqnarray*}
    \frac{\f_1(x^{t+1}(\w^\text{ideal})) - \f_1(x^t) }{\gamma} \leq \frac{-\langle{ \nabla \f_1(x^t),\widetilde g^{t}} \rangle }{\widetilde b_{t+1}} +\frac{ \gamma  L}{2{\widetilde b_{t+1}}^2} \|\widetilde g^{t}\|^2  
    \\
    = -\frac{\|\nabla \f_1(x^t)\|^2}{\widetilde b_{t+1}} +\frac{ \langle{ \nabla \f_1(x^t), \nabla \f_1(x^t) - \widetilde g^{t} \rangle} }{\widetilde b_{t+1}}+\frac{ \gamma  L\|\widetilde g^{t}\|^2}{2{\widetilde b_{t+1}}^2}.
\end{eqnarray*}
Applying~\eqref{eq:aux_approx} and~\eqref{eq:aux_opt}, we get the following inequality: 
\begin{eqnarray*}
    \frac{\f_1(\x^{t+1}) - \delta - \f_1(x^t) }{\gamma} \leq -\frac{\|\nabla \f_1(x^t)\|^2}{b_{t+1}} +\frac{ \langle{ \nabla \f_1(x^t), \nabla \f_1(x^t) - g^{t} \rangle} }{b_{t+1}}+\frac{ \gamma  L\|g^{t}\|^2}{2b_{t+1}^2}.
\end{eqnarray*}
Then, following the same steps as in the main proof from \cite{ward2018adagrad}, we derive 
\begin{align*}
      \min_{t \leq T} \left( \mathbb{E} \left[ \|\nabla \f_1(x^t)\|^{
   \frac{4}{3}} \right] \right)^\frac{3}{2}  \leq  \left( \frac{2b_0}{T} + \frac{4(G+\sigma_*)}{\sqrt{T}}\right)C_{F},
\end{align*}
%\end{strip}
%\lipsum[1] 
where $$C_{F} = \frac{\delta T}{\gamma} + \frac{ \f_1(x^0) - \f_1(x^*) }{\gamma}+ \frac{4\sigma_*+\gamma L}{2}\log \left(\frac{20T\left( \sigma_*^2 + G^2\right)}{b_0^2} +10 \right).$$
This finishes the proof.
\end{proof}

\clearpage
\section{Accelerating MeritOpt}

In this section, we outline preliminary experiments designed to accelerate our approach. We limit the training to 24 hours or to 180K iterations due to time and resource constraints. Moreover, XXX... The simple heuristics are as follows:
\begin{itemize}
    \item \algname{CT + MeritOpt}: MeritOpt is applied once the CT stage reaches its peak performance or computational limit. This heuristic aims to refine the model after it has acquired sufficient knowledge of the target language in default setting.
  
    \item \algname{MeritOpt-Drop}: During training, a language is dropped at the end of an epoch if its weight falls below a predefined threshold. This heuristic is intended to avoid unnecessary computations for languages that do not contribute to model improvement. In our experiments, the threshold was set to 0.15 with 5 languages and 0.2 with 4 languages.
  
    \item \algname{MeritOpt-Cycle}: MeritOpt is applied selectively at certain epochs, while other epochs are trained solely with top-1 weighted language, determined with MeritOpt. This heuristic seeks to introduce MeritOpt intermittently, guiding the model updates towards a more beneficial direction by leveraging multiple languages in a controlled manner.
\end{itemize}

\begin{table}[t]
\centering
\begin{tabular}{cccccccccc}
\hline
\multirow{2}{*}{\textbf{Setting}} & \multicolumn{3}{c}{\textbf{North Sami}} & \multicolumn{3}{c}{\textbf{Java}} & \multicolumn{3}{c}{\textbf{Tagalog}} \\
\cmidrule(lr){2-4} \cmidrule(lr){5-7} \cmidrule(lr){8-10} 
 & \textbf{Score} & \textbf{Time} & \textbf{Iters} & \textbf{Score} & \textbf{Time} & \textbf{Iters} & \textbf{Score} & \textbf{Time} & \textbf{Iters} \\
\midrule
CT & 50.19 & 13H & 180K & 21.04 & 2H & 37K & 33.12 & 9H & 149K \\
\midrule
\algname{MeritOpt} & 41.85 & 24H & 16K & 21.12 & 13H & 3.5K & 32.56 & 24H & 10K \\
\algname{CT + MeritOpt} & 50.95 & 25H & 186K & 21.23 & 8H & 39K & 33.65 & 29H & 101K \\
\algname{MeritOpt-Drop} & 41.43 & 24H & 13K & 20.64 & 2H & 1K & 33.46 & 24H & 13K \\
\algname{MeritOpt-Cycle} & 49.72 & 24H & 180K & 21.31 & 14H & 60K & 32.29 & 13H & 58K \\
\bottomrule
\end{tabular}
\caption{Performance comparison for different settings across languages}
\label{table:comparison_reformatted}
\end{table}

In our experiments, we observed that the performance of the combined \algname{CT + MeritOpt} setting, though noisy and showing minor improvements, did not consistently outperform the individual approaches. For instance, while spBLEU improved slightly, the fluctuations were significant, making the gains less reliable. \algname{MeritOpt} setting showed stable improvement in the last 5K iterations, but performance plateaued. In contrast, the \algname{MeritOpt-Cycle} setting reached a plateau quickly, and the last 50K iterations offered no further gains, with SME still dominating in terms of weight importance.

For the Java dataset, the \algname{CT + MeritOpt} setup showed unpredictable and noisy behavior, with some checkpoints marginally outperforming the original, but with only minor improvements. Interestingly, the weight distribution across languages remained similar to that of the \algname{MeritOpt} setting, except for Java, which exhibited more stability as it was already well-learned. In the \algname{MeritOpt-Drop} setting, Java and other languages were quickly removed, which led to a sharp decline in performance, likely due to the exclusion of target languages, despite the validation loss remaining stable.

In the Tagalog experiments, the \algname{CT + MeritOpt} setup demonstrated steady spBLEU improvements. In the \algname{MeritOpt-Drop} setting, Tagalog and other languages were removed later in training, yet the model’s performance remained stable, likely due to the retention of the target language. Interestingly, this setting yielded results that outperformed both  \algname{MeritOpt} and CT, although it still fell short of the \algname{CT + MeritOpt} approach. Finally, in the \algname{MeritOpt-Cycle} setting, spBLEU fluctuated but stabilized at a reasonable level. While the \algname{MeritOpt} mechanism caused notable gains and losses, all languages held the top-1 position at some point during training, which added unpredictability.

These results highlight that dropping languages can lead to overfitting, particularly if the target language is excluded. The Tagalog case appears to be unique, suggesting that permanently disabling languages could be beneficial only when we are certain they are no longer necessary. Moreover, the \algname{MeritOpt-Cycle} approach often re-prioritized already well-learned languages, potentially hindering performance. 
We suggest that loss averaging by batch could be weighted based on the significance of samples as calibrated by \algname{MeritOpt}, reducing the contribution of less important samples. Lastly, our results indicate that \algname{CT + MeritOpt} offers limited gains, likely because the model converged to the suboptimal local minima, however plain \algname{MeritOpt} converges differently.

\end{document}